\documentclass{article}

\PassOptionsToPackage{numbers, compress}{natbib}
\usepackage[preprint]{neurips_2026}

\usepackage[utf8]{inputenc} %
\usepackage[T1]{fontenc}    %
\usepackage[hidelinks]{hyperref}       %
\usepackage{url}            %
\usepackage{booktabs}       %
\usepackage{amsfonts}       %
\usepackage{graphicx}       %
\usepackage{float}          %
\usepackage{amsmath}        %
\usepackage{amsthm}         %
\usepackage{amssymb}        %
\usepackage{thm-restate}

\newtheoremstyle{boldnote}%
  {}{}{\itshape}{}{\bfseries}{.}{ }%
  {\thmname{#1}\thmnumber{ #2}\thmnote{ \textbf{(#3)}}}
\theoremstyle{boldnote}

\newtheorem{theorem}{Theorem}

\newtheorem{definition}{Definition}

\newtheorem{atheorem}{Theorem}[section]
\newtheorem{alemma}[atheorem]{Lemma}
\newtheorem{aproposition}[atheorem]{Proposition}
\newtheorem{acorollary}[atheorem]{Corollary}

\usepackage{nicefrac}       %
\usepackage{microtype}      %
\usepackage{enumitem}       %
\usepackage{xcolor}         %
\usepackage[normalem]{ulem} %
\usepackage{titletoc}       %
\definecolor{otblue}{RGB}{31,119,180}
\definecolor{indorange}{RGB}{190,130,60}
\definecolor{lowrankblue}{HTML}{5C7FDD}
\definecolor{identorange}{HTML}{E99B3C}
\definecolor{gapgreen}{HTML}{1CA550}
\definecolor{condblue}{HTML}{486FD8}
\definecolor{blindorange}{HTML}{DC8746}

\title{What Time Is It? How Data Geometry Makes Time Conditioning Optional for Flow Matching}

\author{%
  Alec Helbling\thanks{Corresponding author: \texttt{alechelbling@gatech.edu}} \\
  Georgia Tech \\
  \And
  Sebastian Gutierrez Hernandez \\
  Georgia Tech \\
  \And
  Benjamin Hoover \\
  Georgia Tech \& IBM Research \\
  \And
  Duen Horng Chau \\
  Georgia Tech \\
  \And
  Parikshit Ram \\
  IBM Research \\
}

\begin{document}

\maketitle

\begin{abstract}
Recent work has shown that models flow matching models can be trained without explicit time conditioning, challenging the standard view that the interpolation time is needed to disambiguate velocity targets. But why should a time-blind model work at all? Decomposing the time-blind flow matching loss, we identify two sources of irreducible error: a \textit{coupling variance}, which arises from ambiguous velocity targets induced by how noise and data points are paired, and the \textit{time-blindness gap}, which is the additional error caused by ignoring time. This gap shows that time-blind training is strictly harder than conventional training, reinforcing the puzzle that time-blind models work so well in practice. We resolve this tension by showing that the geometry of high-dimensional data makes time identifiable directly from noisy observations. When data concentrates near a $k$-dimensional subspace, time can be recovered from the statistical structure of noisy interpolants in directions orthogonal to the data; under a spiked-covariance model, this yields a closed-form estimator that recovers $t$ from a single observation $z$ at rate $O(1/\sqrt{d-k})$ for ambient dimension $d$. As a consequence, we prove that the time-blindness gap is asymptotically negligible relative to the coupling variance. We empirically demonstrate our identifiability result on real-world data and show that changing the coupling has a much larger effect on loss and sample quality than removing time conditioning across CIFAR-10, CelebA-HQ, and FFHQ. These results explain why time-blind flow matching works and show that the main practical lever is the choice of coupling, not explicit time conditioning.
\end{abstract}

\section{Introduction}

\begin{figure}[t]
\centering
\includegraphics[width=\linewidth]{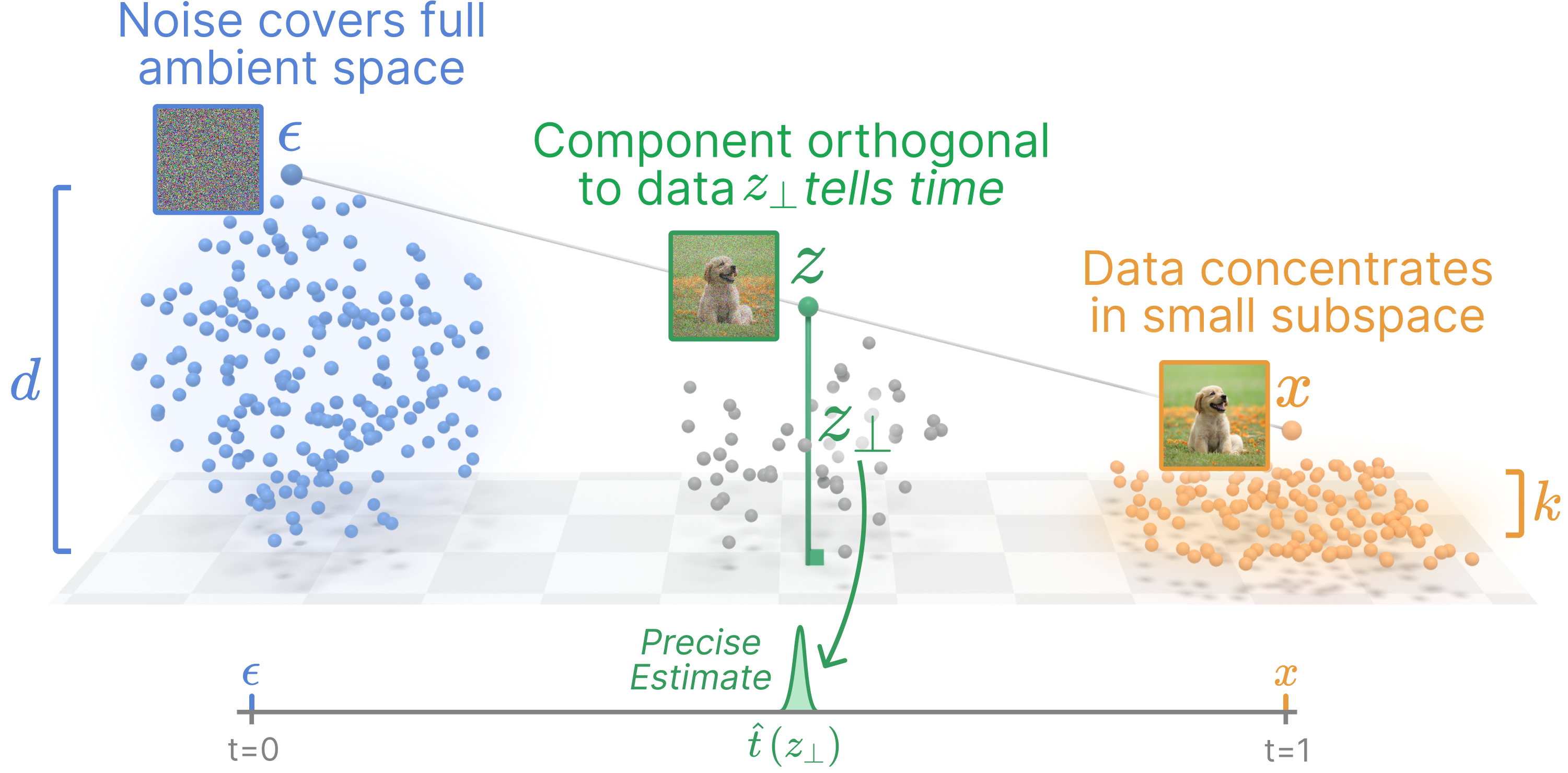}
\vspace{-1.5em}
\caption{\textbf{Time-blind flow matching works because data and noise occupy subspaces of very different sizes, leaving a large residual noise subspace whose variance encodes the interpolation time.} \textcolor{lowrankblue}{(Left) Noise dominates in a large subspace}: the source $\varepsilon \sim \mathcal{N}(0, I_d)$ is isotropic over the full ambient dimension $d$. \textcolor{identorange}{(Right) Data concentrates in a small subspace}: the target distribution is supported on a $k$-dimensional signal subspace with $k \ll d$. \textcolor{gapgreen}{(Center) Identify time from noise-subspace statistics}: the interpolant $z = (1-t)\varepsilon + tx$ inherits the dimensionality mismatch, so its component $z_\perp$ orthogonal to the data subspace is supplied almost entirely by the noise term and is approximately Gaussian with variance directly tied to $t$. Estimating that variance recovers $t$ from a single observation $z$ with no time conditioning required.}
\label{fig:time_identifiability}
\vspace{-1em}
\end{figure}

Flow matching~\citep{lipman2023flowmatchinggenerativemodeling,albergo2023buildingnormalizingflowsstochastic} has emerged as a leading framework for training generative models, learning a time-conditioned velocity field $v_\theta(z,t)$ that transports noise to data along interpolated paths. 
The time variable is usually treated as essential, under the intuition that conditioning on $t$ seems necessary to disambiguate between velocity targets when at the same spatial location.
From this perspective, a time-blind vector field $f_\theta(z)$ should be fundamentally disadvantaged, since it must assign a single velocity to locations that may be reached by many trajectories at different times (Figure~\ref{fig:geometric_intersections}).

Recent work challenges this view. The performance of time-unconditional generative models often barely degrades at all when explicit conditioning is removed and sometimes even improves ~\citep{sun2025noiseconditioning,wang2025equilibriummatchinggenerativemodeling}, and blind denoising analyses show that the noise level can sometimes be inferred from the noisy observation itself under low-dimensional data structure~\citep{kadkhodaie2026blind}. But for flow matching, this is not the whole story. Flow matching has another source of irreducible ambiguity: the coupling between noise and data during training. Even if time is known, many coupled pairs can pass through the same interpolant and induce different velocities. Thus the central question is not if a time-blind model can infer $t$, but rather if the residual cost of hiding $t$ is large compared to the coupling-induced error that all flow matching models already pay.

The time-blind flow matching loss decomposes into a \textit{learnable approximation} term, a \textit{coupling variance} that remains even for an optimal time-conditioned predictor, and a \textit{time-blindness gap} caused by averaging over unobserved times. The gap measures exactly how much harder time-blind training is than standard flow matching, and we show that in practical settings it is small. When data concentrates in a low-dimensional subspace, the interpolant's component orthogonal to that subspace is supplied almost entirely by the Gaussian source, and its magnitude changes predictably with $t$, a built-in statistical clock for the hidden time. We formalize this with an explicit estimator whose error shrinks with the residual ambient dimension, making the cost of hiding time negligible compared with the coupling variance. Coupling, not time conditioning, is the dominant lever.

In summary, our contributions are as follows:

\begin{enumerate}[leftmargin=1.5em, itemsep=0.4em]
\item \textbf{An explanatory loss decomposition for time-blind flow matching.}
We decompose the time-blind flow matching (FM) loss into a learnable error, an irreducible coupling variance shared with standard FM, and an additional time-blindness gap unique to time-blind flows (Proposition ~\ref{prop:three_term}).

\item \textbf{A constructive time-identification statistic for noisy interpolants.}
Under a spiked-covariance data model, the residual-subspace energy of a single interpolant $z$ yields a closed-form estimator of $t$ with error $O_p(1/\sqrt{d-k})$ (Theorem~\ref{thm:t_identifiable}).

\item \textbf{A relative bound showing that coupling dominates the cost of hiding time.}
Combining the decomposition with the identifiability rate, we prove that the time-blindness gap is smaller than the coupling variance by a factor $O(1/(d-k))$ (Theorem~\ref{thm:time_independence_cost}).

\item \textbf{Empirical validation on real data and trained flow models.}
On CIFAR-10, CelebA-HQ, and FFHQ, the residual-subspace statistic recovers interpolation time at the predicted rate (Section \ref{sec:exp_identifiability}), and trained models show that coupling has a much larger effect on loss and sample quality than explicit time conditioning (Section \ref{sec:exp_decomp})
\end{enumerate}

\section{Problem Setup}

\begin{figure}[t]
\centering
\includegraphics[width=\linewidth]{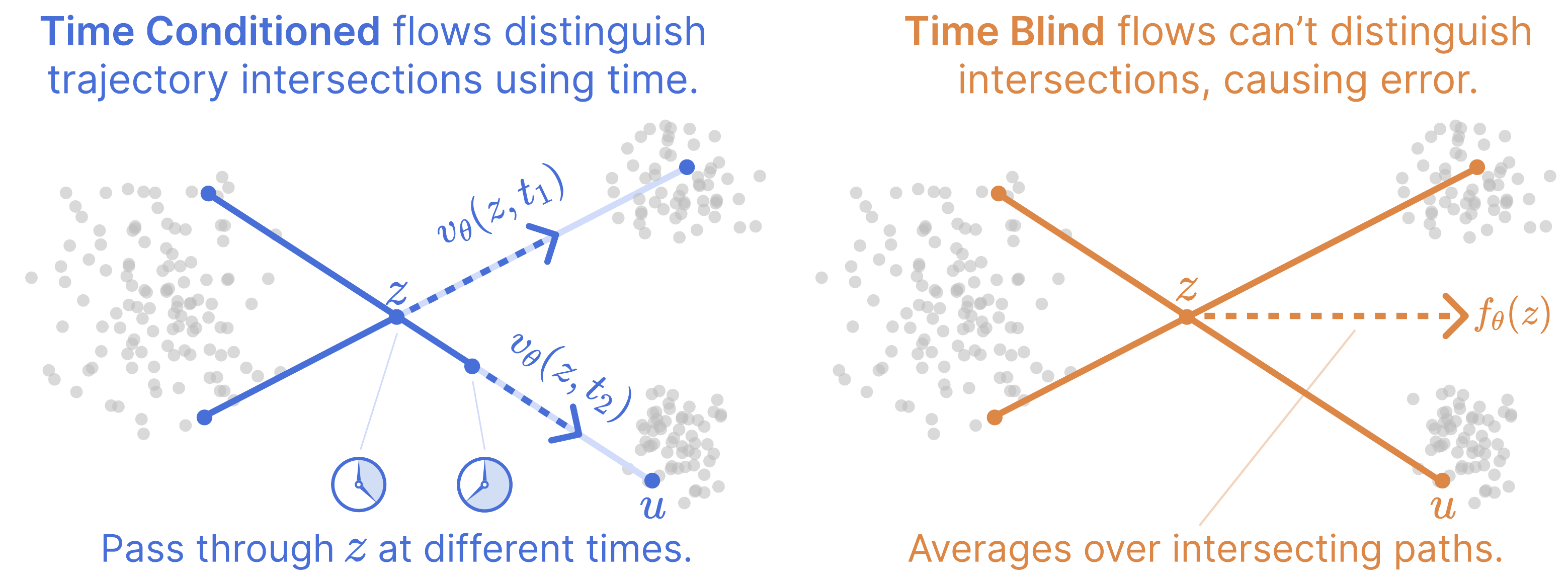}
\vspace{-1em}
\caption{\textbf{\textcolor{condblue}{Time-conditioned} flows disambiguate trajectory intersections using $t$. \textcolor{blindorange}{Time-blind} flows must commit to a single velocity at each location, averaging over conflicting trajectories.} (\emph{Left}) A \textcolor{condblue}{time-conditioned} model $v_\theta(z, t)$ observes $t$ and can assign different velocities at the same spatial location $z$ for different times; here, $v_\theta(z, 0.3)$ and $v_\theta(z, 0.7)$ point along distinct trajectories that pass through $z$. (\emph{Right}) A \textcolor{blindorange}{time-blind} model has no $t$ to disambiguate with, so it must commit to a single $f_\theta(z)$ that averages over all paths through $z$.}
\label{fig:geometric_intersections}
\end{figure}

In this section we present the theoretical formulation for understanding the core problems faced by conventional flow matching models, and more specifically by time-blind models that drop the time conditioning.

\paragraph{Flow matching.}
Flow matching (FM)~\citep{lipman2023flowmatchinggenerativemodeling,albergo2023buildingnormalizingflowsstochastic} learns a time-conditioned velocity field $v_\theta(x, t)$ transporting a noise distribution $p_0$ to a data distribution $p_1$ over $t \in [0, 1]$. Given a coupling $(\varepsilon, x) \sim \pi$ with $\varepsilon \sim p_0$ and $x \sim p_1$, the linear interpolant defines a conditional path
\begin{equation}\label{eq:interpolant}
z = (1-t)\,\varepsilon + t\,x
\end{equation}
with constant conditional velocity $u = x - \varepsilon$. The model is trained to regress $v_\theta(z, t)$ onto $u$:
\begin{equation}
\mathcal{L}_{\mathrm{FM}} = \mathbb{E}_{t,\, (\varepsilon, x) \sim \pi} \left[\left\| v_\theta(z, t) - u \right\|^2 \right].
\end{equation}

\paragraph{Time-blind flow matching.}
We consider a time-blind vector field $f_\theta(z)$ trained against the same interpolant-based regression target $u = x - \varepsilon$. Given $(\varepsilon, x) \sim \pi$ and the interpolant $z = (1-t)\varepsilon + tx$ with $t \sim \mathcal{U}[0, 1]$, the time-blind loss is
\begin{equation}\label{eq:ti_loss}
  \mathcal{L}_{\mathrm{TB}} = \mathbb{E}_{t,\, (\varepsilon, x) \sim \pi} \left[\left\| f_\theta(z) - u \right\|^2 \right].
\end{equation}
The key question is: what is the cost of dropping the time conditioning that FM relies on?

\section{An Explanatory Decomposition of the Time-Blind Loss}
\label{sec:bv_decomp}

The time-blind loss $\mathcal{L}_{\mathrm{TB}}$ mixes together errors of very different origins: some that any flow matching model must pay regardless of time conditioning, some that are unique to dropping $t$, and some that the practitioner can drive away by training a better model. Pulling these apart is what lets us pose the central question of the paper precisely: how does the extra cost of hiding $t$ compare to the irreducible cost that flow matching pays anyway? Proposition~\ref{prop:three_term} makes this separation exact.

\subsection{The Three-Term Decomposition}
\label{sec:three_term_decomp}

We consider arbitrary time-blind vector fields $f : \mathbb{R}^d \to \mathbb{R}^d$. Throughout, unsubscripted $\mathbb{E}$ denotes expectation over the full joint $(\varepsilon, x, t) \sim \pi \otimes \mathcal{U}[0,1]$.

\begin{restatable}[Anatomy of the Time-Blind Loss]{proposition}{threeterm}
\label{prop:three_term}
For any predictor $f : \mathbb{R}^d \to \mathbb{R}^d$, the time-blind loss decomposes as
\begin{equation}\label{eq:three_term}
\mathcal{L}_{\mathrm{TB}}(f) \;=\;
  \underbrace{\mathbb{E}\bigl[\|f(z) - \mathbb{E}[u \mid z]\|^2\bigr]}_{\text{I. Learnable error}}
  \;+\;
  \underbrace{\mathbb{E}\bigl[\mathrm{Var}(u \mid z, t)\bigr]}_{\text{II. Coupling variance }\mathcal{C}_\pi}
  \;+\;
  \underbrace{\mathbb{E}\bigl[\mathrm{Var}\bigl(\mathbb{E}[u \mid z, t] \,\big|\, z\bigr)\bigr]}_{\text{III. Time-Blindness Gap }\Delta}.
\end{equation}
\textit{(Proof in Appendix~\ref{app:coupling_error}.)}
\end{restatable}

\noindent The decomposition is exact, holds for any $f$, and requires no assumptions on the coupling or model class. Each term has a distinct origin and a distinct lever, which we now unpack in turn.

\paragraph{I. Learnable error.}
This is the squared distance between the predictor and the optimal time-blind predictor $f^*(z) := \mathbb{E}[u \mid z]$. It is the only term that depends on the choice of predictor, and the only one the practitioner can drive to zero: with sufficient capacity and data, $f_\theta \to f^*$ and this term vanishes.

\paragraph{II. Coupling variance.}
This is the irreducible variance
\[
\mathcal{C}_\pi \;:=\; \mathbb{E}\bigl[\mathrm{Var}(u \mid z, t)\bigr]
\]
of the regression target after conditioning on \emph{both} $z$ and $t$. This floor exists even if we grant the model the very information we are trying to hide: fixing $(z, t)$ does not fix the underlying pair $(\varepsilon, x)$, since $z = (1-t)\varepsilon + tx$ is one affine constraint on two variables. Distinct pairs that route through the same $(z,t)$ generally carry different velocities, and the best any predictor can do is predict their conditional mean $v^*(z,t) := \mathbb{E}[u \mid z, t]$, leaving residual variance $\mathcal{C}_\pi$. This floor is paid by the optimal time-\emph{conditioned} predictor in standard flow matching too (Lemma~\ref{lem:fm_bias_variance}), and the time-blind model must at least inherit it. The size of $\mathcal{C}_\pi$ is controlled entirely by the coupling $\pi$: independent couplings permit many distinct $(\varepsilon, x)$ to route through the same $(z, t)$ with disagreeing velocities, while structured couplings such as mini-batch optimal transport pair nearer points and tighten the conditional. This is the lever we manipulate in Section~\ref{sec:ot_coupling}.

\paragraph{III. Time-blindness gap.}
This is the \emph{additional} price
\[
\Delta \;:=\; \mathbb{E}\bigl[\mathrm{Var}\bigl(v^*(z,t) \mid z\bigr)\bigr]
\]
paid for not observing $t$, on top of the coupling variance. A time-blind predictor cannot assign different velocities to the same location at different times (Figure~\ref{fig:geometric_intersections}), so at its optimum it must average $v^*$ over the posterior $t \mid z$: $f^*(z) = \mathbb{E}[v^*(z, t) \mid z]$. The gap $\Delta$ measures how much $v^*$ varies across this posterior --- if $z$ nearly determines $t$, the gap collapses. Since $\Delta \geq 0$ always (Corollary~\ref{cor:eqm_geq_fm}), this term is unique to time-blind models and makes the time-blind optimum at least as hard as conventional flow matching.

\section{Data Geometry Mitigates the Time-Blindness Gap}
\label{sec:identifiability}

Proposition~\ref{prop:three_term} reduces the puzzle to a comparison between two irreducible terms: the time-blindness gap $\Delta$ and the coupling variance $\mathcal{C}_\pi$. This section shows that, in high-dimensional low-rank regimes, the hidden time is already encoded in the geometry of the interpolant. In particular, the residual subspace provides a high-dimensional clock that identifies $t$ from $z$ alone. Once $t$ is nearly determined by $z$, explicitly conditioning on $t$ adds little information, so the time-blindness gap is small. Under the spiked covariance model, this intuition yields the asymptotic comparison $\Delta / \mathcal{C}_\pi = O(1/(d-k))$, leaving the coupling variance as the dominant loss floor.

\begin{figure}[t]
\centering
\includegraphics[width=\linewidth]{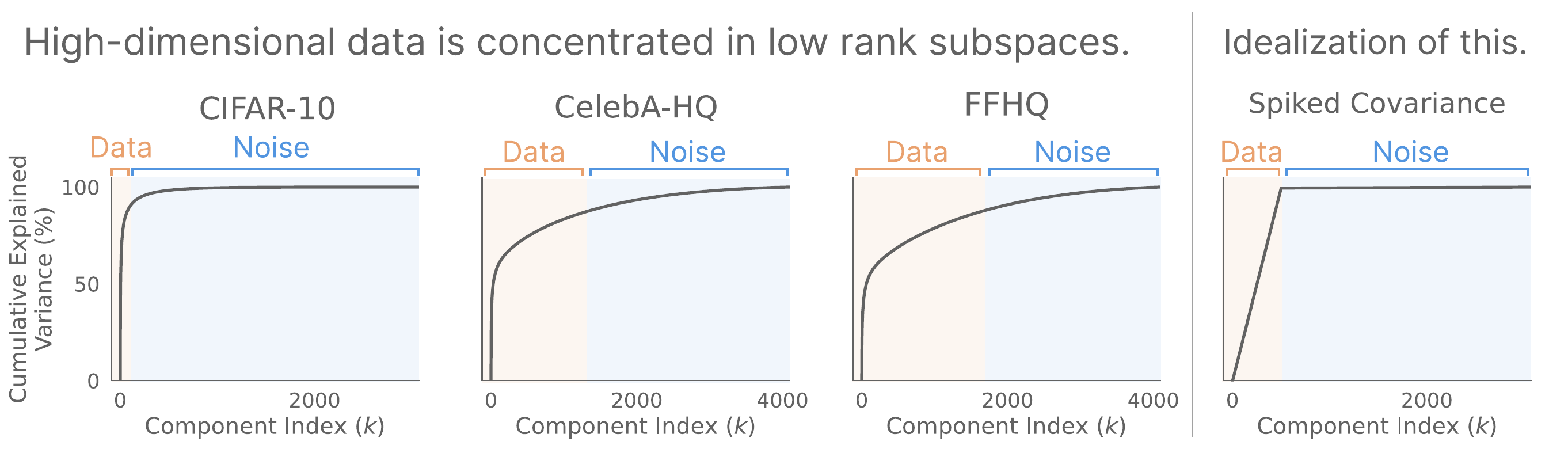}
\vspace{-1em}
\caption{\textbf{Real data concentrates in a low-rank subspace and leaves a sizable residual subspace.} Cumulative explained variance captured by the top-$k$ principal components on each real dataset and on the spiked covariance model used in Theorem~\ref{thm:t_identifiable}. Pixel-space CIFAR-10 shows dramatic low-rank concentration; on the FFHQ and CelebA-HQ VAE latents variance is more spread out across components but still concentrates most of its mass in a small fraction of them, leaving a large residual subspace ($d - k$ in the thousands at the $80\%$ threshold shown). The spiked model is one tractable instance of the broader condition that $d-k$ is large enough for the noise-subspace clock.}
\label{fig:explained_variance}
\end{figure}

\subsection{The Noise Subspace is a High-Dimensional Clock}
\label{sec:picture}

Suppose the data $x \in \mathbb{R}^d$ concentrates near a $k$-dimensional signal subspace with $k \ll d$.
Let $U_\perp \in \mathbb{R}^{d \times (d-k)}$ be an orthonormal basis for the orthogonal complement of this subspace, and let $P_\perp = U_\perp U_\perp^\top$ be the projection preserved by that basis.
The key observation is that $P_\perp z$ removes most of the data signal while retaining the ambient Gaussian source noise.
Thus the residual energy $\|P_\perp z\|^2$ changes predictably with the interpolation time $t$.

\paragraph{Spiked covariance model.}
To make this mechanism explicit, we analyze a standard spiked covariance model~\citep{johnstone2001largest, chen2023score, wang2024unreasonable}
\[
x \sim \mathcal{N}(0,\, U_x \Lambda U_x^\top + \sigma^2 I_d),
\]
where $U_x \in \mathbb{R}^{d \times k}$ spans the signal subspace, $\Lambda$ contains the excess variances in the signal directions, and $\sigma^2 I_d$ is an isotropic residual floor.
This model idealizes the empirical observation that natural data often has most of its variance in a low-dimensional subspace while retaining small residual variance in the remaining directions (Figure~\ref{fig:explained_variance}).

\paragraph{The residual variance clock.}
Under this model, projection onto the orthogonal complement of $U_x$ eliminates the spike directions.
In each residual coordinate, $P_\perp z$ combines an ambient Gaussian-noise coordinate with variance $1$, scaled by $(1-t)$, and a residual data coordinate with variance $\sigma^2$, scaled by $t$.
Since these terms are independent and centered, their variances add:
\[
\operatorname{Var}\!\left[(P_\perp z)_j\right]
=
(1-t)^2 \operatorname{Var}(\varepsilon_j)
+
t^2 \operatorname{Var}(x_j)
=
(1-t)^2 + t^2\sigma^2.
\]
Therefore,
\begin{equation}\label{eq:projected_gaussian}
P_\perp z \;\sim\; \mathcal{N}\!\bigl(0,\; \sigma_\perp^2(t)\, I_{d-k}\bigr),
\qquad
\sigma_\perp^2(t) \;:=\; (1-t)^2 + t^2 \sigma^2.
\end{equation}
We call $\sigma_\perp^2(t)$ the \emph{variance clock}. A high-dimensional isotropic Gaussian places almost all its mass on a thin shell of squared radius $(d-k)\,\sigma_\perp^2(t)$, with relative fluctuations of order $1/\sqrt{d-k}$, so observing the empirical radius of $P_\perp z$ pinpoints the variance and hence $t$.
This gives a closed-form estimator on any branch where the clock is invertible:
\begin{equation}\label{eq:t_hat_def}
\widehat{\sigma}_\perp^2 \;:=\; \frac{1}{d-k}\,\|P_\perp z\|^2,
\qquad
\hat{t}(z) \;:=\; (\sigma_\perp^2)^{-1}\!\bigl(\widehat{\sigma}_\perp^2\bigr).
\end{equation}

Because the shell concentrates at rate $1/\sqrt{d-k}$, this single-sample estimator recovers $t$ with vanishing error as the residual dimension grows, as shown in Theorem~\ref{thm:t_identifiable}.

\begin{theorem}[Identifiability of $t$ from $z$; informal]
\label{thm:t_identifiable}
Under the spiked covariance model, on any interval where the noise-subspace variance clock identifies time, the estimator $\hat{t}(z)$ in~\eqref{eq:t_hat_def} satisfies
\[
\hat{t}(z) \;=\; t \;+\; O_p\!\left(\frac{1}{\sqrt{d-k}}\right).
\]
\textit{(Formal statement and proof in Appendix~\ref{app:low_rank_concentration}.)}
\end{theorem}

\noindent Thus, removing explicit time conditioning does not necessarily remove time information. In high-dimensional low-rank regimes, the interpolant $z$ itself contains an implicit clock, so a sufficiently expressive time-blind predictor can infer the relevant time from geometry rather than from an external conditioning variable. The only caveat is that this clock is branchwise: $\sigma_\perp^2(t)$ decreases until its critical point $t^* = 1/(1+\sigma^2)$ and increases thereafter, so the variance alone identifies $t$ only on intervals where this map is one-to-one. In our datasets, $t^*$ lies near the data endpoint, so the descending branch covers nearly the full interpolation path (Table~\ref{tab:dataset_sigma2}).

\paragraph{Beyond isotropic spectra.}
The residual-clock mechanism does not rely on an exactly isotropic residual spectrum.
For anisotropic Gaussian residuals, concentration is governed by the effective rank
\[
r_{\mathrm{eff}}(\Sigma_{P,t}) \;:=\; \frac{(\operatorname{tr}\Sigma_{P,t})^2}{\operatorname{tr}(\Sigma_{P,t}^2)},
\qquad
\Sigma_{P,t} \;:=\; (1-t)^2 I + t^2 \Sigma_P,
\]
which measures the number of projected directions that contribute appreciable variance.
Thus the role played by $d-k$ in the spiked model is played more generally by $r_{\mathrm{eff}}(\Sigma_{P,t})$ (Proposition~\ref{prop:general_cov_linear}, Appendix~\ref{app:effective_rank_generalization}).
We focus on the spiked model in the main text because it yields a closed-form inverse clock
and a transparent comparison between the time-blindness gap and coupling variance
(Theorem~\ref{thm:time_independence_cost}).
Empirically, PCA-fitted residual spectra on real data have sufficiently large effective rank
for the same prediction to track the observed estimator-error distributions
(Section~\ref{sec:exp_identifiability}).

\subsection{The Time-Blindness Gap is Negligible Relative to Coupling Variance}
\label{sec:cost_time_independence}

Theorem~\ref{thm:t_identifiable} shows that $t$ can be recovered from $z$ alone at rate $O_p(1/\sqrt{d-k})$.
This directly controls the time-blindness gap in Proposition~\ref{prop:three_term}.
Indeed, if $t$ is nearly determined by $z$, then the time-aware optimal predictor $v^*(z,t)=\mathbb{E}[u\mid z,t]$ varies little over the posterior uncertainty $t\mid z$.
Equivalently, explicitly revealing $t$ adds little information beyond what is already present in $z$.
Under the spiked covariance model, this intuition yields a quantitative comparison between the extra cost of hiding time and the coupling variance already paid by time-conditioned flow matching.

\begin{theorem}[Cost of Time-Independence; informal]
\label{thm:time_independence_cost}
Under the spiked covariance model with independent coupling, on any interval where the noise-subspace variance clock identifies time, the time-blindness gap $\Delta$ and coupling variance $\mathcal{C}_\pi$ satisfy
\[
    \frac{\Delta}{\mathcal{C}_\pi}
    \;=\; \frac{\mathbb{E}\bigl[\mathrm{Var}\bigl(\mathbb{E}[u \mid z, t] \mid z\bigr)\bigr]}{\mathbb{E}\bigl[\mathrm{Var}(u \mid z, t)\bigr]}
    \;=\; O\!\left(\frac{1}{d-k}\right).
\]
\textit{(Formal statement and proof in Appendix~\ref{app:relative_time_independence_cost}.)}
\end{theorem}

\noindent
The rate is set by the noise-subspace dimension alone: holding $k$ fixed and growing $d$ gives the $O(1/d)$ scaling, while holding $d-k$ fixed and growing $k$ leaves the ratio bounded by a constant.
This separates our account from purely ambient-dimension explanations of time identifiability~\citep{sun2025noiseconditioning}, and we validate both regimes empirically in Sections~\ref{sec:exp_ambient_dimension} and~\ref{sec:exp_decomp}.

\paragraph{Coupling as the dominant lever.}
The coupling variance $\mathcal{C}_\pi$ is not specific to time-blind models: by Lemma~\ref{lem:fm_bias_variance}, it is the irreducible floor paid even by the optimal time-conditioned predictor $v^*(z,t)=\mathbb{E}[u\mid z,t]$.
Theorem~\ref{thm:time_independence_cost} shows that the additional price of removing time conditioning is lower order in the high-dimensional low-rank regime.
Thus, once time is identifiable from $z$, reducing $\mathcal{C}_\pi$ should matter more than restoring explicit time conditioning.
\label{sec:ot_coupling}%
In our experiments, we reduce $\mathcal{C}_\pi$ using mini-batch optimal transport~\citep{tong2024improvinggeneralizingflowbasedgenerative,fatras2021minibatch}, which approximates the population OT map~\citep{brenier1991polar} by solving a linear assignment within each training batch.

\section{Experiments}
\label{sec:experiments}

Our experiments confirm three core predictions of our theory. Section~\ref{sec:exp_decomp} confirms that the three-term decomposition is additive across paired (time conditioning, coupling) settings, with the coupling variance dominating the loss and mini-batch OT translating that dominance into sample-quality gains. Section~\ref{sec:exp_identifiability} confirms that the noise-subspace estimator $\hat{t}(z)$ matches its predicted error distribution on synthetic and real data. Section~\ref{sec:exp_identifiability} shows that on the spiked Gaussian, the time-blindness gap is governed by the noise-subspace dimension $d-k$ rather than $d$ alone, and is dwarfed by the coupling variance across the $(d,k)$ plane.

We evaluate on two complementary settings: (1)~\textbf{Real image datasets.} CIFAR-10~\citep{Krizhevsky09learningmultiple} in pixel space ($d{=}3{,}072$) with an ADM UNet, and CelebA-HQ~\citep{karras2018progressivegrowinggansimproved} and FFHQ~\citep{karras2019stylebasedgeneratorarchitecturegenerative} in Stable Diffusion VAE~\citep{rombach2022highresolutionimagesynthesislatent} latent space ($d{=}4{,}096$) with a SiT-B/2 backbone~\citep{ma2024sitexploringflowdiffusionbased}, spanning two modalities and two architectures so that observed effects are not artifacts of either. (2)~\textbf{Synthetic spiked covariance.} The Gaussian model of Section~\ref{sec:picture} with $\Lambda$ eigenvalues spaced linearly in $[1, 10]$ and $\sigma^2 = 0.1$, where varying $d$ and $k$ directly sweeps the noise-subspace dimension $d-k$ that governs the rate in Theorem~\ref{thm:t_identifiable}. (Details in Appendix~\ref{app:datasets}.)

\begin{figure}[t]
\centering
\includegraphics[width=\linewidth]{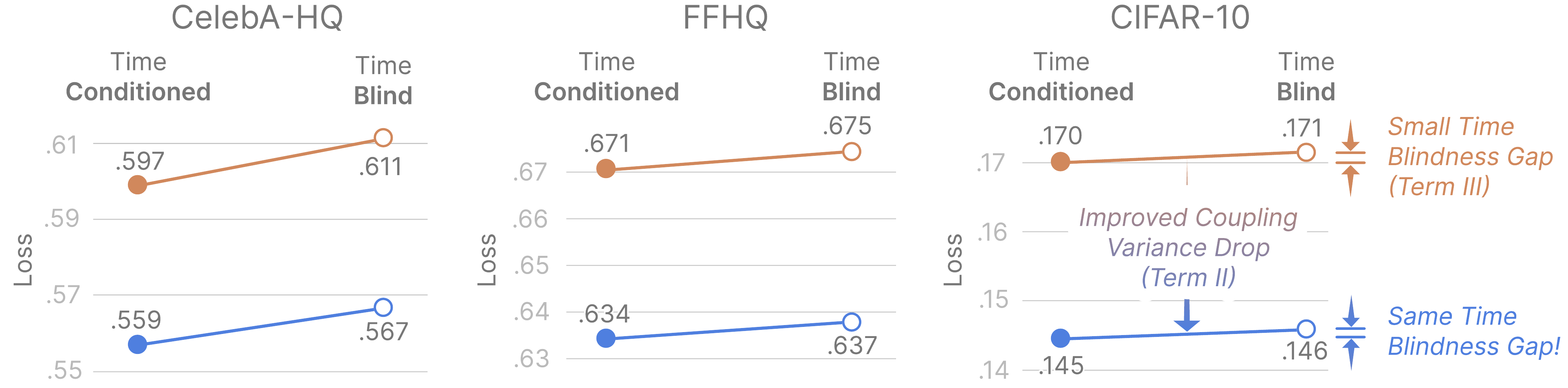}
\vspace{-1em}
\caption{\textbf{The error decomposition is additive and coupling-dominated.} Shown are converged test-set MSE for each ablation. \textcolor{condblue}{OT coupling, time-conditioned}~(\textcolor{condblue}{$\bullet$}); \textcolor{condblue}{OT coupling, time-blind}~(\textcolor{condblue}{$\circ$}); \textcolor{blindorange}{naive coupling, time-conditioned}~(\textcolor{blindorange}{$\bullet$}); \textcolor{blindorange}{naive coupling, time-blind}~(\textcolor{blindorange}{$\circ$}). The large vertical drop between lines is the coupling variance (Term~II); the small horizontal gap within each line is the time-blindness gap (Term~III). On both datasets the coupling effect is consistent regardless of time conditioning (additivity), and the time-blindness gap is negligible.}
\label{fig:loss_decomposition}
\end{figure}

\subsection{The Decomposition is Additive and Coupling is the Dominant Lever}
\label{sec:exp_decomp}

We train flow matching models on each dataset, varying (i)~whether the model conditions on time $t$ or receives $t = 0$ at every step, and (ii)~whether the coupling is independent or mini-batch OT (exact EMD solver). Each model is evaluated under its own training coupling, so the measured loss reflects the irreducible floor of the coupling it was trained with (details in Appendix~\ref{app:exp_loss_decomp}). Figure~\ref{fig:loss_decomposition} and Table~\ref{tab:main_results} confirm three predictions:

\begin{enumerate}[leftmargin=1.2em, labelsep=0.4em, itemsep=2pt, topsep=2pt]
\item \textbf{Coupling and time conditioning appear to be distinct, approximately additive levers.} If coupling and time conditioning act independently, switching from independent to OT coupling should produce a similar improvement with and without time conditioning. Across all three datasets, switching from independent to OT coupling produces nearly the same improvement in both the time-conditioned and time-blind rows. This suggests that improved coupling does not simply compensate for missing time information, and that removing time conditioning does not erase the benefit of OT coupling.

\item \textbf{Coupling variance dominates the loss floor.} The time-blindness gap should be small relative to the coupling effect on real, high-dimensional data (Section~\ref{sec:cost_time_independence}). On CIFAR-10 the coupling effect is roughly $25\times$ the time effect (dropping time conditioning costs $<\!0.7\%$ of total loss); on CelebA-HQ and FFHQ the time-blindness gap accounts for $<\!2\%$ of total loss.

\item \textbf{OT coupling improves both loss and sample quality.} Converged FID-50k mirrors the loss decomposition: OT coupling improves FID on every dataset regardless of whether the model is time-conditioned or time-blind, while the difference between time-conditioned and time-blind variants remains small (Table~\ref{tab:main_results}, Appendix~\ref{app:exp_loss_decomp}).
\end{enumerate}

\subsection{Validating Time Identifiability on Real Data}
\label{sec:exp_identifiability}

Theorem~\ref{thm:t_identifiable} predicts that the interpolation time $t$ is recoverable from $z$ alone at rate $O(1/\sqrt{d-k})$, with errors approximately Gaussian. We construct the noise-subspace estimator $\hat{t}(z)$ from the proof and compare its empirical error distribution against the theoretical Gaussian on the synthetic spiked covariance model and on each of our three real datasets, fitting the spiked covariance via PCA at the $95\%$ explained-variance threshold for the theoretical prediction.

\paragraph{Fitting the spiked covariance to real data.} Given a dataset we compute the sample covariance $\hat{\Sigma}$ and eigendecompose it as $\hat{\Sigma} = V\,\mathrm{diag}(\hat{\lambda}_1, \dots, \hat{\lambda}_d)\,V^\top$ with eigenvalues sorted descending. At a chosen rank $k$ we set the signal subspace $U_x$ to the top-$k$ eigenvectors, the isotropic floor $\sigma^2 = \frac{1}{d-k}\sum_{j>k}\hat{\lambda}_j$ to the mean of the trailing eigenvalues, and the spike eigenvalues $\Lambda_i = \max(\hat{\lambda}_i - \sigma^2,\,0)$ for $i \le k$, so that $U_x\Lambda U_x^\top + \sigma^2 I_d$ recovers the empirical top-$k$ spectrum and collapses the residual tail into a single isotropic term. These parameters enter only the theoretical Gaussian prediction; the data $x$ in $z = (1-t)\varepsilon + tx$ are actual samples drawn from the dataset, not from this Gaussian.

\noindent The theorem makes four confirmed predictions (details in Appendix~\ref{app:exp_identifiability}):

\begin{figure}[!t]
\centering
\includegraphics[width=\linewidth]{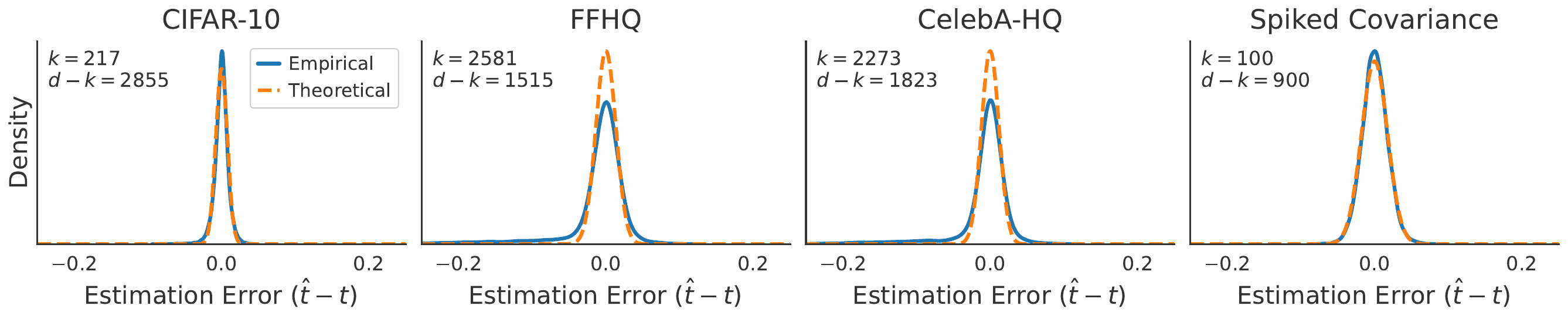}
\vspace{-1.5em}
\caption{\textbf{Empirical estimates of $t$ on real datasets are precise and match the theoretically predicted bound.} Distribution of estimation error $\hat{t} - t$ for the noise-subspace estimator across four settings: a synthetic spiked covariance model, CIFAR-10 (pixel space), FFHQ (VAE latent), and CelebA-HQ (VAE latent), with $k$ chosen at the $95\%$ explained-variance threshold per dataset. Empirical (solid) closely matches the theoretical Gaussian prediction (dashed) from the spiked covariance bound (Theorem~\ref{thm:t_identifiable}).}
\label{fig:t_estimation}
\vspace{-1em}
\end{figure}

\begin{enumerate}[leftmargin=1.2em, labelsep=0.4em, itemsep=2pt, topsep=2pt]
\item \textbf{Predictions are tight in absolute terms.} With $d-k$ in the thousands, the empirical estimates track closely around the true time on every panel. The empirical histograms are sharply peaked at zero and the spread shrinks as $d-k$ grows, consistent with the $O(1/\sqrt{d-k})$ rate predicted by Theorem~\ref{thm:t_identifiable}.

\item \textbf{The closed-form estimator achieves the predicted rate on both synthetic and real data.} Under the spiked Gaussian where Theorem~\ref{thm:t_identifiable}'s assumptions hold exactly, the empirical histogram of $\hat{t}-t$ tracks the theoretical Gaussian closely. The same prediction transfers to real, non-Gaussian data: fitting the spiked model via PCA yields a theoretical Gaussian that matches the empirical error distribution on CIFAR-10, FFHQ, and CelebA-HQ. The bound requires only that $d-k$ be large, not strict low-rank concentration, so even on VAE latents whose spectra are visibly flatter than pixel CIFAR-10 (Figure~\ref{fig:explained_variance}), $d-k$ remains in the thousands and the estimator achieves the predicted rate.

\item \textbf{Approximate invertibility is not a major concern in practice.} Because $t^*$ sits close to $1$ on every dataset, the descending-branch inverse covers nearly the full interpolation path; even when $t$ is sampled uniformly over $[0, 1]$, the mean absolute error remains small on every dataset ($0.006$ on CIFAR-10, $0.021$ on FFHQ, $0.018$ on CelebA-HQ; Appendix~\ref{app:exp_clock_critical_point}).
\end{enumerate}

\subsection{Closed-Form Validation of the Variance Decomposition}
\label{sec:exp_ambient_dimension}

To isolate the dimensional scaling predicted by our theory from confounds of training and finite-sample noise, we evaluate the three loss terms in an idealized spiked-covariance setting where each admits a closed-form or low-variance numerical estimator. Our analysis makes three sharp predictions about how the ambient dimension $d$ and signal dimension $k$ shape the three terms of Proposition~\ref{prop:three_term}. We test all three on the spiked Gaussian (details in Appendix~\ref{app:exp_time_gap}). Figure~\ref{fig:ambient_dim_analysis} reports each along a complementary slice of the $(d, k)$ grid.

\begin{enumerate}[leftmargin=1.2em, labelsep=0.4em, itemsep=2pt, topsep=2pt]

\item
\textbf{Shrinking the noise subspace inflates the gap.} At fixed $d$, growing $k$ shrinks the $(d-k)$-dimensional noise subspace through which $t$ leaks into $z$, so the time-blindness gap should explode. At $d=1024$, sweeping $k$ blows up the gap (Figure~\ref{fig:ambient_dim_analysis}, left).

\item
\textbf{Enlarging the ambient space collapses the gap.} At fixed $k$, growing $d$ enlarges the noise subspace and Theorem~\ref{thm:time_independence_cost} predicts the ratio $\Delta / \mathcal{C}_\pi$ scales as $1/(d-k)$. At $k=64$, sweeping $d$ drives the gap toward zero while Terms~I and III stay bounded (Figure~\ref{fig:ambient_dim_analysis}, right).

\item
\textbf{Coupling variance dominates the gap across the plane.} On the bulk of the $(d, k)$ plane (anywhere off the $k \approx d$ diagonal), the coupling variance should dwarf the gap. The heatmap (center) is light everywhere except along the diagonal, and in both slice plots the coupling curve sits well above the gap throughout.
\end{enumerate}

\begin{figure}[t]
\centering
\includegraphics[width=\linewidth]{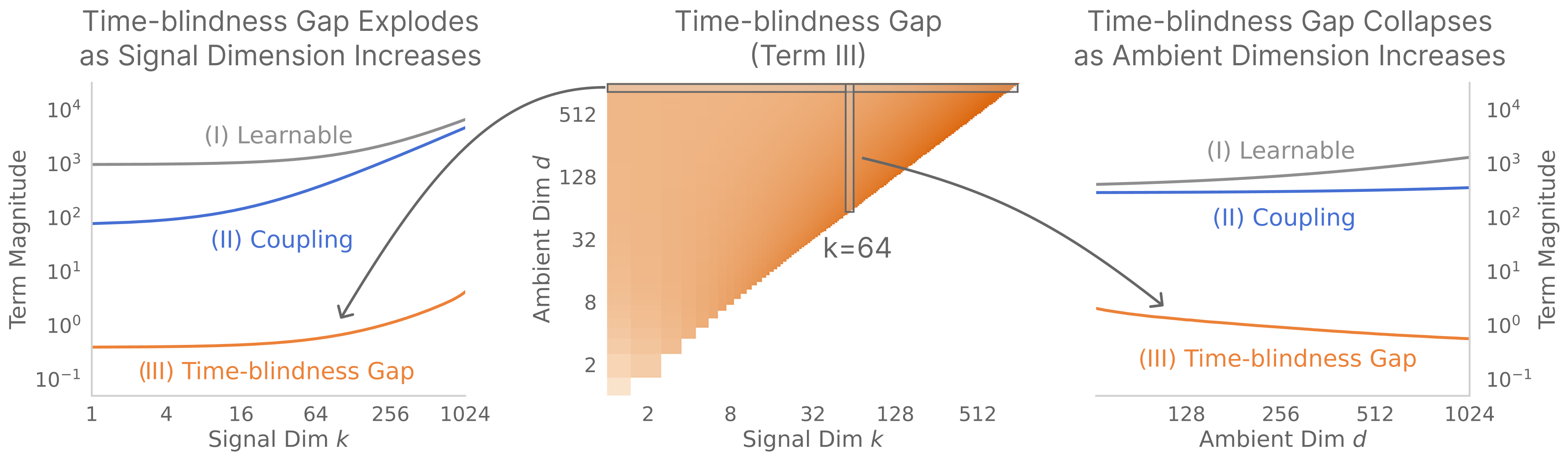}
\vspace{-1em}
\caption{\textbf{The time-blindness gap is governed by $d-k$ and dwarfed by the coupling variance.} (Left) Growing $k$ at fixed $d=1024$ inflates both the gap and the coupling loss. (Center) Heatmap of the gap magnitude across the $(d,k)$ plane. (Right) Growing $d$ at fixed $k=64$ collapses the gap.}
\label{fig:ambient_dim_analysis}
\end{figure}

\section{Related Work}
\label{sec:related_work}
\paragraph{Noise-level estimation and blind denoising.}
Recovering the noise level from a corrupted observation is a classical problem in image processing, with estimators based on local image statistics, patch structure, and low-rank covariance structure~\citep{shin2005block,liu2013single,chen2015efficient}. Modern blind denoisers instead learn this behavior from data, including convolutional, self-supervised, and real-image blind-denoising methods~\citep{guo2019toward,batson2019noise2self,zhang2023blind}. These works motivate the possibility that an explicit noise or time label may be redundant when the corrupted sample itself contains enough information to infer it.

\paragraph{Blind generative models and intrinsic dimension.}
Recent work has extended this question to generative modeling. \citet{kadkhodaie2026blind} show that blind denoising diffusion models can follow an implicit noise schedule under low intrinsic dimension, building on the view that natural images concentrate near low-dimensional geometric structure~\citep{henaff2014local,pope2021intrinsic,brown2023verifying}. Related perspectives connect denoising to score estimation and empirical Bayes identities~\citep{miyasawa1961empirical,efron2011tweedie,raphan2011least}, and to autonomous or time-unconditional generative dynamics~\citep{sahraeeardakan2026geometrynoisediffusionmodels,wang2025equilibriummatchinggenerativemodeling}. Our setting differs from these works because flow matching contains an additional ambiguity: even if time is known, the target velocity remains random through the coupling between source and target samples.

\paragraph{Noise conditioning, residual dimension, and coupling in flow matching.}
Closest to our motivation, \citet{sun2025noiseconditioning} analyze why noise-unconditional denoising and flow-based generative models can remain effective, arguing that the posterior over the hidden noise level concentrates in high ambient dimension. We refine this dimensional picture for low-rank data in flow matching: the identifiable clock lives in the residual noise subspace orthogonal to the data, so the relevant dimension is $d-k$, not $d$ alone. This distinction matters when the intrinsic dimension grows with the ambient dimension. In parallel, prior work has shown that flow-matching performance depends strongly on the source--target coupling, with improved couplings based on rectified flows, optimal transport, and minibatch optimal transport~\citep{liu2022flowstraightfastlearning,tong2024improvinggeneralizingflowbasedgenerative,fatras2021minibatch}. Our contribution is to disentangle these effects: we compare the time-blindness gap against the irreducible coupling variance and show why coupling can remain the dominant lever even when time is hidden.

\section{Discussion}
\label{sec:discussion}

Our results show that time conditioning is not always essential in flow matching: in high-dimensional settings, the noisy interpolant can itself reveal the hidden time, because residual noise directions act like an implicit clock. Once time can be inferred from the sample, the dominant source of error is the ambiguity induced by how noise and data are paired, so coupling matter more than explicit time conditioning. This also points toward more powerful time-independent generative models, including equilibrium-style models that learn autonomous dynamics without an externally supplied time variable. Several directions remain open. Our theory identifies one mechanism for time recovery, but it remains to validate which mechanisms trained networks actually use. Future work could probe whether learned models rely on residual noise energy, data-subspace structure, or other implicit time signals. The theory could also be refined beyond the spiked-covariance setting toward nonlinear data geometry, anisotropic covariance, learned representations, or statistics that exploit both residual noise directions and information inside the data subspace.

\section*{Acknowledgements}

This work is in part supported by the National Science Foundation Graduate Research Fellowship and NSF grant 2403297. We would also like to thank Vikranth Srivatsa for his feedback on the manuscript.

\bibliographystyle{plainnat}
\bibliography{references}

\newpage
\appendix
\counterwithin{table}{section}
\counterwithin{figure}{section}

\section*{Appendix Contents}
\startcontents[appendix]
\printcontents[appendix]{l}{1}{\setcounter{tocdepth}{2}}
\newpage

\section{Proofs for Section~\ref{sec:bv_decomp}: Isolating Sources of Error}\label{app:error_sources}

\subsection{Anatomy of the Time-Blind Loss}\label{app:coupling_error}

To isolate the effect of the coupling cleanly, we analyze the time-blind loss over all vector fields $f:\mathbb{R}^d\to\mathbb{R}^d$. We combine the standard bias--variance decomposition with the law of total variance to obtain a single three-term decomposition.

\threeterm*

\begin{proof}
\textbf{Step 1: Bias--variance decomposition.}
By the standard bias--variance decomposition (conditioning on $z$), the cross term $\mathbb{E}[(f(z) - \mathbb{E}[u \mid z])^\top(\mathbb{E}[u \mid z] - u) \mid z] = 0$ vanishes, giving
\begin{equation}\label{eq:bv_intermediate}
\mathcal{L}_{\mathrm{TB}}(f) = \underbrace{\mathbb{E}[\| f(z) - \mathbb{E}[u \mid z] \|^2]}_{\text{Reducible (bias)}} + \underbrace{\mathbb{E}[\mathrm{Var}(u \mid z)]}_{\text{Irreducible (variance)}}.
\end{equation}

\textbf{Step 2: Law of total variance.}
We now split the irreducible term $\mathbb{E}[\mathrm{Var}(u \mid z)]$ by conditioning further on $t$. By the law of total variance,
\[
\mathrm{Var}(u \mid z) = \mathbb{E}[\mathrm{Var}(u \mid z, t) \mid z] + \mathrm{Var}(\mathbb{E}[u \mid z, t] \mid z).
\]
Taking the outer expectation over $z$:
\[
\underbrace{\mathbb{E}[\mathrm{Var}(u \mid z)]}_{\text{Irreducible variance}} = \underbrace{\mathbb{E}[\mathrm{Var}(u \mid z, t)]}_{\text{Coupling variance } \mathcal{C}_\pi} + \underbrace{\mathbb{E}[\mathrm{Var}(\mathbb{E}[u \mid z, t] \mid z)]}_{\text{Time-Blindness Gap } \Delta}.
\]
Since $u = x - \varepsilon$, conditioning on $(z, t)$ is equivalent to FM conditioning on $(z_t, t)$ with the same target, so $\mathcal{C}_\pi = \mathbb{E}[\mathrm{Var}(u_t \mid z_t, t)]$. Substituting into~\eqref{eq:bv_intermediate} gives the three-term decomposition.
\end{proof}

For use throughout the appendix, write
\[
    \Delta \;:=\; \mathbb{E}[\mathrm{Var}(\mathbb{E}[u \mid z, t] \mid z)],
    \qquad
    \mathcal{C}_\pi \;:=\; \mathbb{E}[\mathrm{Var}(u \mid z, t)],
\]
matching the underbrace labels above. Step 1 shows that the irreducible variance $\mathbb{E}[\mathrm{Var}(u \mid z)] = \mathcal{C}_\pi + \Delta$ depends only on the coupling $\pi$, so the only way to reduce it is to change the coupling. Step 2 further splits this into $\mathcal{C}_\pi$ and the time-blindness gap $\Delta$. The next subsection formalizes the comparison with flow matching.

\subsection{Time-Blind Flow Matching is at Least as Hard as Flow Matching}\label{app:fm_bias_variance}

The FM loss admits an analogous bias--variance decomposition. The coupling variance term (II) in Proposition~\ref{prop:three_term} corresponds exactly to the FM irreducible error floor.

\begin{alemma}[Bias--Variance Decomposition of FM Loss]\label{lem:fm_bias_variance}
Let $z_t = (1-t)\varepsilon + tx$ and $u_t = x - \varepsilon$ under a fixed joint distribution of $(\varepsilon, x)$, and let $v : \mathbb{R}^d \times [0,1] \to \mathbb{R}^d$ be any predictor. Then the FM loss decomposes as
\[
\mathcal{L}_{\mathrm{FM}}(v) := \mathbb{E}[\| v(z_t, t) - u_t \|^2]
= \underbrace{\mathbb{E}\big[\| v(z_t, t) - \mathbb{E}[u_t \mid z_t, t] \|^2\big]}_{\text{Reducible (bias)}}
+ \underbrace{\mathbb{E}[\mathrm{Var}(u_t \mid z_t, t)]}_{\text{Irreducible (variance)}}.
\]
In particular, $\min_v \mathcal{L}_{\mathrm{FM}}(v) = \mathbb{E}[\mathrm{Var}(u_t \mid z_t, t)]$, achieved by $v^*(z_t, t) = \mathbb{E}[u_t \mid z_t, t]$.
\end{alemma}

\begin{proof}
The proof is identical to Step~1 of Proposition~\ref{prop:three_term}, with conditioning on $(z_t, t)$ instead of $z$ alone.
\end{proof}

Combining the FM floor with the three-term decomposition immediately gives:

\begin{acorollary}[Time-blind irreducible error $\geq$ FM]\label{cor:eqm_geq_fm}
The irreducible time-blind error is at least as large as the FM irreducible error:
\[
    \min_f \mathcal{L}_{\mathrm{TB}}(f)
    \;=\; \mathcal{C}_\pi + \Delta
    \;\geq\; \underbrace{\mathcal{C}_\pi}_{\min_v \mathcal{L}_{\mathrm{FM}}(v)}.
\]
The gap $\Delta \ge 0$ is the cost of time-independence: it vanishes exactly when the time-aware optimal velocity $\mathbb{E}[u \mid z, t]$ is, almost surely, already determined by $z$.
\end{acorollary}

When the data distribution has low intrinsic dimensionality $k$ relative to the ambient dimension $d$, $\Delta$ is bounded by $O(1/(d-k))$ times $\mathcal{C}_\pi$; we formalize this in Appendix~\ref{app:relative_time_independence_cost} (Theorem~\ref{thm:relative_time_independence_cost}). Crucially, this geometric mitigation leaves $\mathcal{C}_\pi$ unaddressed: it remains an independent, controllable source of difficulty that motivates the OT coupling strategy developed in this work.

\subsection{Relating Irreducible Variance and the Straightness of Paths}\label{app:straightness}

\begin{definition}[Straightness~{\citep{liu2022flowstraightfastlearning}}]\label{def:straightness}
Let $z_t$ be the trajectory of an ODE with velocity field $v$ and endpoints $(\varepsilon, x) \sim \pi$. The \emph{straightness defect} is
\begin{equation}
    S \;:=\; \int_0^1 \mathbb{E}\big[\|(x - \varepsilon) - v(z_t, t)\|^2\big]\,dt.
\end{equation}
$S = 0$ iff all trajectories are constant-velocity straight lines, in which case single-step Euler integration is exact.
\end{definition}

Applied to the optimal time-blind velocity $f^*(z) = \mathbb{E}[u \mid z]$, the straightness defect equals the irreducible variance:
\begin{equation}\label{eq:auto_straightness_app}
S_{\mathrm{auto}} \;:=\; \int_0^1 \mathbb{E}\bigl[\|(x - \varepsilon) - f^*(z_t)\|^2\bigr]\,dt \;=\; \mathbb{E}\bigl[\mathrm{Var}(u \mid z)\bigr],
\end{equation}
since $f^*(z_t) = \mathbb{E}[u \mid z_t]$ makes the integrand independent of $t$. By the law of total variance (Proposition~\ref{prop:three_term}),
\begin{equation}\label{eq:curvature_decomp}
S_{\mathrm{auto}} \;=\; \underbrace{\mathcal{C}_\pi}_{S_{\mathrm{FM}}} \;+\; \Delta,
\end{equation}
where $S_{\mathrm{FM}} = \int_0^1 \mathbb{E}[\|(x - \varepsilon) - v^*(z_t, t)\|^2]\,dt$ is the straightness defect of \citet{liu2022flowstraightfastlearning} applied to the optimal FM velocity, and $\Delta$ is the time-blindness gap.

\section{Proofs for Section~\ref{sec:identifiability}: Closing the Two Sources of Irreducible Error}\label{app:mitigations}

Proposition~\ref{prop:three_term} establishes that the irreducible error of time-blind flow matching is strictly higher than that of FM, raising the natural question of why time-blind models work at all in practice. Extending the analysis of \citet{kadkhodaie2026blind} from the diffusion setting to linear interpolants, we resolve this puzzle by showing that the interpolation time $t$ is almost entirely determined by the observation $z$ alone, so conditioning on $t$ provides little additional information and the gap $\Delta = \mathbb{E}[\operatorname{Var}(\mathbb{E}[u \mid z, t] \mid z)]$ becomes negligible relative to the irreducible coupling variance $\mathcal{C}_\pi$.

\paragraph{Mechanism: concentration reveals time.}
The observation $z_t = (1-t)\varepsilon + t\,x$ mixes noise and signal with weights determined by $t$. In high dimensions, simple statistics of $z_t$ (such as its projected norm) concentrate sharply around deterministic functions of $t$. If this concentration is sufficiently strong, $t$ can be recovered from $z$ alone, so conditioning on $t$ provides little additional information.

\subsection{Baseline Spiked Model and Notation}\label{app:spiked_model_notation}

The headline chain of this appendix, spiked identifiability (Subsection~\ref{app:low_rank_concentration}) $\Rightarrow$ absolute gap bound (Subsection~\ref{app:absolute_time_blindness_gap}) $\Rightarrow$ relative cost (Subsection~\ref{app:relative_time_independence_cost}), is built on a single baseline regime that we collect once here for reuse.

\paragraph{Spiked Gaussian model.}
The data follow a spiked Gaussian:
\begin{equation}
    x \sim \mathcal{N}(0, \Sigma),
    \qquad
    \Sigma = U_x \Lambda U_x^\top + \sigma^2 I_d,
\end{equation}
where $U_x \in \mathbb{R}^{d \times k}$ spans the $k$-dimensional signal subspace, $\Lambda = \mathrm{diag}(\lambda_1, \dots, \lambda_k)$ collects the signal variances with $\lambda_{\max} := \max_i \lambda_i < \infty$, and $\sigma^2 > 0$ gives isotropic residual variance in the orthogonal complement. This is a standard tractable proxy for the broader empirical observation that natural high-dimensional data concentrate near a low-dimensional subspace, used throughout high-dimensional statistics~\citep{johnstone2001largest} and the theoretical analysis of diffusion models~\citep{chen2023score, wang2024unreasonable}. Because all results are rotationally invariant, we may take $\Sigma$ diagonal in its eigenbasis with per-coordinate variances $S_i = \lambda_i$ for $i \le k$ and $S_i = \sigma^2$ for $i > k$.

\paragraph{Linear interpolant.}
We observe
\begin{equation}\label{eq:linear_interpolant_app}
    z_t = (1-t)\,\varepsilon + t\,x,
    \qquad \varepsilon \sim \mathcal{N}(0, I_d),
    \qquad t \sim \mathrm{Unif}(I),
\end{equation}
where $I \subset (0, 1)$ is a closed interval avoiding $t^* = 1/(1+\sigma^2)$.

\paragraph{Noise-subspace statistic.}
Let $P_\perp \in \mathbb{R}^{d \times d}$ project onto the $(d-k)$-dimensional orthogonal complement of the signal subspace. The key statistic is
\begin{equation}\label{eq:sigma_hat_perp}
    \widehat{\sigma}_\perp^2(z) \;:=\; \frac{1}{d-k}\,\|P_\perp z\|^2,
\end{equation}
which estimates the deterministic noise-subspace variance
\begin{equation}\label{eq:sigma_perp}
    \sigma_\perp^2(t) \;:=\; (1-t)^2 + t^2 \sigma^2.
\end{equation}
We also write $r_i(t) = (1-t)^2 + t^2 S_i$ for the per-coordinate interpolant variance and $\alpha_i(t) := (tS_i - (1-t))/r_i(t)$ for the per-coordinate optimal slope.

\paragraph{Noise-subspace time estimator.}
Inverting the variance clock $\sigma_\perp^2$ yields the noise-subspace time estimator used throughout the appendix:
\begin{equation}\label{eq:t_hat_app}
    \hat t(z) \;:=\; \bigl(\sigma_\perp^2\bigr)^{-1}\!\Bigl(\Pi_{\sigma_\perp^2(I)}\!\bigl(\widehat\sigma_\perp^2(z)\bigr)\Bigr),
\end{equation}
where $\Pi_{\sigma_\perp^2(I)}$ denotes projection onto the closed interval $\sigma_\perp^2(I) \subset \mathbb{R}$. The projection ensures $\hat t(z) \in I$ even when $\widehat\sigma_\perp^2(z)$ leaves the range of $\sigma_\perp^2$ on $I$ (a positive-probability event, since the underlying chi-square statistic is unbounded), and only sharpens the bounds below since projection cannot increase distance to any point in the interval.

\paragraph{Constants convention.}
All implicit constants in the $O(\cdot)$ statements that follow depend only on the time interval $I$, the noise floor $\sigma^2$, and the bounded signal spectrum $\lambda_{\max}$, not on $d$ or $k$.

\subsection{Identifiability under Spiked Covariance}\label{app:low_rank_concentration}

This appendix restates Theorem~\ref{thm:t_identifiable} from the main text in full formal detail and provides the complete proof. Working in the baseline spiked Gaussian model (Subsection~\ref{app:spiked_model_notation}), the concentration argument here takes a particularly clean form: projecting onto the noise subspace yields an isotropic Gaussian, so the projected norm follows a standard $\chi^2_{d-k}$ distribution (up to scaling), and the estimation rate scales with the noise-subspace dimension $d-k$.

\noindent\textbf{Theorem~\ref{thm:t_identifiable} (Restated).}\textit{ Under the baseline spiked model of Subsection~\ref{app:spiked_model_notation}, the estimator
\[
\hat{t}(z) \;=\; (\sigma_\perp^2)^{-1}\!\bigl(\widehat{\sigma}_\perp^2(z)\bigr)
\]
satisfies $\hat{t}(z) = t + O_p\!\left(1/\sqrt{d-k}\right)$ for $t$ bounded away from $t^* = 1/(1+\sigma^2)$.}

\begin{proof}
Write
\[
    \sigma_\perp^2(t) := (1-t)^2 + t^2\sigma^2.
\]
We show that the projected noise-subspace variance acts as a high-dimensional clock for the interpolation time. The proof has three steps: projection, concentration, and inversion.

\textbf{Step 1: Projecting onto the noise subspace produces an isotropic Gaussian clock.}
Let $P_\perp$ be the orthogonal projector onto the complement of the $k$-dimensional signal subspace. Projecting the interpolant gives
\[
    P_\perp z_t \;=\; (1-t)\,P_\perp \varepsilon \;+\; t\,P_\perp x.
\]
Under the spiked covariance model,
\[
    P_\perp \varepsilon \sim \mathcal{N}(0, I_{d-k}),
    \qquad
    P_\perp x \sim \mathcal{N}(0, \sigma^2 I_{d-k}),
\]
and these two projected terms are independent. Therefore, conditional on $t$,
\[
    P_\perp z_t \sim \mathcal{N}\!\bigl(0,\, \sigma_\perp^2(t)\, I_{d-k}\bigr).
\]
Thus
\[
    \widehat{\sigma}_\perp^2(z) \;:=\; \frac{1}{d-k}\,\|P_\perp z\|^2
\]
is the empirical second moment of $d-k$ independent Gaussian coordinates, each with variance $\sigma_\perp^2(t)$. This projected variance is the statistic from which we recover time.

\textbf{Step 2: The projected empirical variance concentrates around $\sigma_\perp^2(t)$.}
From the Gaussian representation above,
\[
    \widehat{\sigma}_\perp^2(z)
    \;\overset{d}{=}\;
    \sigma_\perp^2(t)\,\frac{Y}{d-k},
    \qquad
    Y \sim \chi^2_{d-k}.
\]
By standard chi-square concentration,
\[
    \frac{Y}{d-k} \;=\; 1 \;+\; O_p\!\left(\frac{1}{\sqrt{d-k}}\right).
\]
Therefore,
\[
    \widehat{\sigma}_\perp^2(z)
    \;=\; \sigma_\perp^2(t)\!\left(1 + O_p\!\left(\frac{1}{\sqrt{d-k}}\right)\right).
\]
Since $\sigma_\perp^2(t)$ is bounded on any compact interval $I \subset (0,1)$, this gives
\[
    \widehat{\sigma}_\perp^2(z) \;=\; \sigma_\perp^2(t) \;+\; O_p\!\left(\frac{1}{\sqrt{d-k}}\right).
\]

\textbf{Step 3: Inverting the variance clock transfers the rate to $t$.}
The derivative of the variance curve is
\[
    \frac{d}{dt}\sigma_\perp^2(t) \;=\; -2(1-t) + 2t\sigma^2,
\]
and vanishes only at $t^* = 1/(1+\sigma^2)$. Since $I$ is closed and $t^* \notin I$, $|(\sigma_\perp^2)'(t)|$ is bounded below on $I$, so $\sigma_\perp^2$ is monotone there and its inverse is Lipschitz on $\sigma_\perp^2(I)$. Define
\[
    \hat{t}(z) \;:=\; (\sigma_\perp^2)^{-1}\!\bigl(\widehat{\sigma}_\perp^2(z)\bigr).
\] Since $(\sigma_\perp^2)^{-1}\bigl(\sigma_\perp^2(t)\bigr) = t$, Lipschitzness of the inverse gives
\[
    |\hat{t}(z) - t| \;\le\; C\,\bigl|\widehat{\sigma}_\perp^2(z) - \sigma_\perp^2(t)\bigr|.
\]
Using Step 2,
\[
    \hat{t}(z) - t \;=\; O_p\!\left(\frac{1}{\sqrt{d-k}}\right).
\]
Equivalently,
\[
    \hat{t}(z) \;=\; t \;+\; O_p\!\left(\frac{1}{\sqrt{d-k}}\right).
\]
This proves the claimed identifiability rate.
\end{proof}

\subsection{Closed-Form Inversion of the Variance Clock}\label{app:closed_form_inversion}

The proof of Theorem~\ref{thm:t_identifiable} characterizes the estimator $\hat{t}(z) = (\sigma_\perp^2)^{-1}(\hat\sigma_\perp^2)$ abstractly via Lipschitzness of the inverse on a compact interval avoiding $t^*$. For the linear interpolant the inversion has a closed form, which is the formula evaluated in all experiments.

\paragraph{Inverting the variance clock.}
Expanding $(1-t)^2$, collecting like powers, and applying the quadratic formula:
\begin{align}
  \sigma_\perp^2(t)
    &\;=\; (1-t)^2 + t^2\sigma^2 \notag\\
    &\;=\; 1 - 2t + t^2 + t^2\sigma^2 \notag\\
    &\;=\; (1+\sigma^2)\,t^2 - 2t + 1, \notag
  \intertext{so setting $\sigma_\perp^2(t) = \hat\sigma_\perp^2$ and moving everything to one side gives the standard form $(1+\sigma^2)\,t^2 - 2t + (1-\hat\sigma_\perp^2) = 0$, whose two roots are}
  t_\pm
    &\;=\; \frac{2 \pm \sqrt{4 - 4(1+\sigma^2)(1 - \hat\sigma_\perp^2)}}
                {2(1+\sigma^2)}. \label{eq:t_hat_two_roots}
\end{align}

\paragraph{Branch selection.}
The clock $\sigma_\perp^2(t)$ is a strictly convex parabola attaining its global minimum
\[
  \sigma_\perp^2(t^*) \;=\; \frac{\sigma^2}{1+\sigma^2}
  \qquad \text{at} \qquad
  t^* \;=\; \frac{1}{1+\sigma^2}.
\]
Each value $\hat\sigma_\perp^2 \in (\sigma^2/(1+\sigma^2),\, 1]$ has two preimages: $t_-$ on the descending branch $[0, t^*]$ and $t_+$ on the ascending branch $[t^*, 1]$. On any working interval $I \subset [0, t^*]$ (the regime used throughout our experiments, where $\sigma^2 \ll 1$ pushes $t^*$ close to $1$) the inverse selects the descending branch, equivalently the minus sign:
\begin{equation}\label{eq:t_hat_inversion}
  \hat{t}(z)
  \;=\;
  \frac{2 - \sqrt{4 - 4(1+\sigma^2)(1 - \hat{\sigma}_\perp^2)}}
       {2(1+\sigma^2)}.
\end{equation}
This is the closed-form realization of the inverse $(\sigma_\perp^2)^{-1}$ used in Step~3 of the proof of Theorem~\ref{thm:t_identifiable}.

\subsection{Effective-Rank Generalization}\label{app:effective_rank_generalization}

The spiked proof above relies on a specific structural fact: projecting onto the noise subspace yields an isotropic Gaussian, so the projected norm $\|z_\perp\|^2$ is a standard $\chi^2_{d-k}$ up to the scalar $\sigma_\perp^2(t)$. Real data covariances rarely exhibit such a sharp signal/noise split; e.g.\ the latent covariances of trained VAEs typically decay smoothly over several orders of magnitude. We now show that the same project--concentrate--invert argument carries through for arbitrary PSD $\Sigma$ and any orthogonal projector $P$, with the isotropic $\chi^2_{d-k}$ replaced by a weighted chi-square whose effective sample size is the \emph{effective rank} of the projected interpolant covariance.

\paragraph{Effective dimension of the observation.}
The strength of the concentration argument depends on how many independent directions contribute to the statistic. In the isotropic setting of Appendix~\ref{app:low_rank_concentration}, concentration improves with the ambient dimension of the projected subspace, yielding the familiar $1/\sqrt{n}$ scaling for an $n$-dimensional average. In the anisotropic setting, however, different directions contribute unequally, and the relevant notion of dimension is no longer the ambient $n$ but the \emph{effective rank} of the underlying covariance: for a PSD matrix $A \in \mathbb{R}^{n \times n}$ with eigenvalues $\{\lambda_i\}$,
\[
    r_{\mathrm{eff}}(A)
    \;:=\; \frac{(\operatorname{tr} A)^2}{\operatorname{tr}(A^2)}
    \;=\; \frac{\bigl(\sum_i \lambda_i\bigr)^2}{\sum_i \lambda_i^2}.
\]
Unlike the algebraic rank, the effective rank quantifies how evenly variance is distributed across eigen-directions. When all eigenvalues are equal, $r_{\mathrm{eff}}(A) = n$; in highly anisotropic regimes it can be much smaller, reflecting a collapse onto a few dominant directions. It therefore plays the role of an intrinsic statistical dimension governing concentration rates, replacing the ambient dimension in heterogeneous settings.

\paragraph{Setup.}
Let $x \sim \mathcal{N}(0, \Sigma)$ with $\Sigma \in \mathbb{R}^{d \times d}$ an arbitrary PSD covariance, $z_t = (1-t)\varepsilon + tx$ with $\varepsilon \sim \mathcal{N}(0, I_d)$, and let $P$ be an orthogonal projector of rank $m$ with orthonormal basis $Q \in \mathbb{R}^{d \times m}$. Write $\Sigma_P := Q^\top \Sigma Q$ for the induced covariance on $\operatorname{range}(P)$, $\bar\mu_P := \operatorname{tr}(\Sigma_P)/m$ for its spectral mean, and $\Sigma_{P,t} := (1-t)^2 I_m + t^2 \Sigma_P$ for the projected interpolant covariance, whose effective rank we denote $r_{\mathrm{eff}}(P,t) := r_{\mathrm{eff}}(\Sigma_{P,t})$. As in the spiked case, the variance-to-time map $t \mapsto \sigma_t^2(\bar\mu_P) := (1-t)^2 + t^2\bar\mu_P$ has a unique critical point $t^*_{\bar\mu_P} := 1/(1+\bar\mu_P)$.

\begin{aproposition}[Effective-rank identifiability]\label{prop:general_cov_linear}
Under the setup above, on any closed interval avoiding $t^*_{\bar\mu_P}$, the estimator $\hat t(z)$ obtained by inverting $\sigma_t^2(\bar\mu_P) = \|Pz\|^2/m$ on the relevant branch satisfies
\[
    \hat t(z) \;=\; t \;+\; O_p\!\left(r_{\mathrm{eff}}(P,t)^{-1/2}\right).
\]
\end{aproposition}

\begin{proof}
The argument follows the same project--concentrate--invert structure as Theorem~\ref{thm:t_identifiable}; only the concentration step changes.

\textbf{Step 1: Weighted-Gaussian projection.}
Since $P = QQ^\top$, $\|Pz_t\|^2 = \|Q^\top z_t\|^2$, and $Q^\top z_t = (1-t) Q^\top \varepsilon + t\, Q^\top x$ with $Q^\top \varepsilon \sim \mathcal{N}(0, I_m)$ and $Q^\top x \sim \mathcal{N}(0, \Sigma_P)$ independent. Rotating to the eigenbasis of $\Sigma_P$ decomposes $Q^\top z_t$ into independent Gaussian coordinates with per-coordinate variance $\sigma_{t,i}^2 := (1-t)^2 + t^2 \mu_i$, where $\{\mu_i\}_{i=1}^m$ are the eigenvalues of $\Sigma_P$. Writing $\chi^2_{1,i}$ for i.i.d.\ $\chi^2_1$ variates,
\[
    \|Pz_t\|^2 \;\overset{d}{=}\; \sum_{i=1}^m \sigma_{t,i}^2 \,\chi^2_{1,i}.
\]
This is the spiked Step~1 with weights $\sigma_{t,i}^2$ in place of a common scalar $\sigma_\perp^2(t)$, and $\Sigma_{P,t}$ has eigenvalues exactly $\{\sigma_{t,i}^2\}_{i=1}^m$.

\textbf{Step 2: Concentration via the effective-rank identity.}
Let $\hat q := \|Pz\|^2 / m$. Using $\mathbb{E}[\chi^2_1] = 1$ and $\operatorname{Var}(\chi^2_1) = 2$,
\[
    \mathbb{E}[\hat q] \;=\; \tfrac{1}{m}\sum_i \sigma_{t,i}^2 \;=\; \sigma_t^2(\bar\mu_P), \qquad
    \operatorname{Var}(\hat q) \;=\; \frac{2}{m^2}\sum_i (\sigma_{t,i}^2)^2.
\]
Recognising $\sum_i \sigma_{t,i}^2 = \operatorname{tr}\Sigma_{P,t} = m\,\sigma_t^2(\bar\mu_P)$ and $\sum_i (\sigma_{t,i}^2)^2 = \operatorname{tr}(\Sigma_{P,t}^2)$, the definition $r_{\mathrm{eff}}(P,t) = (\operatorname{tr}\Sigma_{P,t})^2 / \operatorname{tr}(\Sigma_{P,t}^2)$ gives $\sum_i (\sigma_{t,i}^2)^2 = m^2 \sigma_t^2(\bar\mu_P)^2 / r_{\mathrm{eff}}(P,t)$, so
\[
    \operatorname{Var}(\hat q) \;=\; \frac{2\,\sigma_t^2(\bar\mu_P)^2}{r_{\mathrm{eff}}(P,t)}.
\]
The effective rank acts as an effective sample size for the heteroskedastic average $\hat q$: it equals $m$ when all $\sigma_{t,i}^2$ are equal and degrades smoothly under spectral anisotropy. Chebyshev's inequality applied to $\hat q$ yields
\[
    \hat q(z) \;=\; \sigma_t^2(\bar\mu_P) \;+\; O_p\!\left(\frac{\sigma_t^2(\bar\mu_P)}{\sqrt{r_{\mathrm{eff}}(P,t)}}\right).
\]
This is the only step where the proof diverges from Theorem~\ref{thm:t_identifiable}: the textbook $\chi^2_{d-k}$ rate $1/\sqrt{d-k}$ is replaced by the weighted-chi-square rate $1/\sqrt{r_{\mathrm{eff}}(P,t)}$, with the latter reducing to the former when the weights are constant.

\textbf{Step 3: Lipschitz inversion.}
The variance-to-time map $t \mapsto \sigma_t^2(\bar\mu_P)$ has derivative $-2(1-t) + 2t\bar\mu_P$, vanishing only at $t^*_{\bar\mu_P}$. On any closed interval $I$ with $t^*_{\bar\mu_P} \notin I$, $|(\sigma_t^2(\bar\mu_P))'|$ is bounded below, so the inverse is Lipschitz on $\sigma_t^2(\bar\mu_P)(I)$ with constant $C$ depending only on $I$ and $\bar\mu_P$. The argument from Step~3 of Theorem~\ref{thm:t_identifiable} then gives
\[
    |\hat t(z) - t| \;\leq\; C\,\bigl|\hat q(z) - \sigma_t^2(\bar\mu_P)\bigr| \;=\; O_p\!\left(r_{\mathrm{eff}}(P,t)^{-1/2}\right),
\]
which is the claimed rate.
\end{proof}

\begin{acorollary}[Spiked recovery]\label{cor:spiked_from_general}
Theorem~\ref{thm:t_identifiable} is the case of Proposition~\ref{prop:general_cov_linear} with $\Sigma = U_x \Lambda U_x^\top + \sigma^2 I_d$ and $P = P_\perp$.
\end{acorollary}

\begin{proof}
Under the spiked covariance and $P = P_\perp$, $P_\perp U_x = 0$, so the induced covariance collapses to the isotropic residual $\Sigma_P = \sigma^2 I_{d-k}$. All eigenvalues $\mu_i$ equal $\sigma^2$, so $\bar\mu_P = \sigma^2$, $t^*_{\bar\mu_P} = 1/(1+\sigma^2) = t^*$, and $\Sigma_{P,t} = \sigma_\perp^2(t)\,I_{d-k}$. Constant eigenvalues give $r_{\mathrm{eff}}(P,t) = (d-k)^2/(d-k) = d-k$, and the Proposition's rate becomes the $O_p(1/\sqrt{d-k})$ bound of Theorem~\ref{thm:t_identifiable}.
\end{proof}

\paragraph{Scope of the spiked model in the rest of the paper.}
Proposition~\ref{prop:general_cov_linear} establishes that the identifiability mechanism is not specific to isotropic residuals; the spiked model provides the cleanest closed-form constants but no essential structure. The remainder of the paper retains the spiked model for two distinct reasons. First, the propagation from identifiability to the time-blindness gap $\Delta$ (Section~\ref{app:absolute_time_blindness_gap}) requires closed-form per-coordinate fourth-moment calculations that the spiked structure makes available; the corresponding anisotropic calculation introduces additional weighted-moment bookkeeping without changing the qualitative conclusion that $\Delta/\mathcal{C}_\pi$ vanishes with the residual-subspace dimension. Second, the empirical predictor in Section~\ref{sec:exp_identifiability} uses the spiked-derived $1/\sqrt{d-k}$ rate because it has a single interpretable parameter ($k$ at a fixed explained-variance threshold) that we report transparently per dataset, and the empirical fit on VAE latents with visibly anisotropic spectra (Figure~\ref{fig:t_estimation}) shows the spiked prediction is already tight without the effective-rank refinement. Proposition~\ref{prop:general_cov_linear} is therefore best read as a robustness statement about the identifiability mechanism, not as a competing predictor or an empirical hypothesis.

\subsection{Absolute Bound on the Time-Blindness Gap}\label{app:absolute_time_blindness_gap}

\begin{aproposition}[Absolute time-blindness gap bound]\label{prop:absolute_time_blindness_gap}
Under the spiked covariance model with $t \sim \mathrm{Unif}(I)$,
\begin{equation}\label{eq:delta_abs_bound}
    \Delta \;=\; O\!\left(\frac{d}{d-k}\right),
\end{equation}
where $\Delta := \mathbb{E}_z\!\left[\mathrm{Var}_{t \mid z}\!\big(\mathbb{E}[u \mid z, t]\big)\right]$.
\end{aproposition}

\begin{proof}
Let
\[
    m \;:=\; d - k,
    \qquad
    q(t) \;:=\; \sigma_\perp^2(t) \;=\; (1-t)^2 + t^2\sigma^2,
    \qquad
    \widehat q(z) \;:=\; \frac{\|P_\perp z\|^2}{m}.
\]
The proof has two steps. First, we upper-bound the time-blindness gap by the error of a concrete time-blind predictor that estimates $t$ from $z$ (a brief inverse-clock argument turns this into the error of the noise-subspace variance estimate $\widehat q$). Second, we bound the resulting weighted variance-estimation error.

\textbf{Step 1: Reduce the gap to time-estimation error.}
Recall that
\[
    \Delta \;=\; \mathbb{E}\!\left[\bigl\|v^*(z,t) - \mathbb{E}[v^*(z,t)\mid z]\bigr\|^2\right],
\]
where $v^*(z,t) = \mathbb{E}[u \mid z, t]$ is the optimal time-aware predictor.

The conditional expectation $\mathbb{E}[v^*(z,t)\mid z]$ is the best predictor of $v^*(z,t)$ among predictors that only depend on $z$. Therefore, for any such predictor $g(z)$,
\[
    \Delta \;\le\; \mathbb{E}\!\left[\|v^*(z,t) - g(z)\|^2\right].
\]
We use the particular predictor
\[
    g(z) \;:=\; v^*(z, \hat t(z)),
\]
where $\hat t(z) \in I$ is the time estimate constructed in Step 2. This gives
\[
    \Delta \;\le\; \mathbb{E}\!\left[\|v^*(z,t) - v^*(z, \hat t(z))\|^2\right].
\]

It remains to control how much $v^*(z,t)$ changes when $t$ is replaced by $\hat t(z)$. In the diagonal Gaussian model,
\[
    v^*_i(z, t) \;=\; \alpha_i(t)\, z_i,
    \qquad
    \alpha_i(t) \;=\; \frac{tS_i - (1-t)}{(1-t)^2 + t^2 S_i}.
\]
On $I$, the functions $\alpha_i$ are uniformly Lipschitz across all coordinates, since $S_i \in [\sigma^2, \lambda_{\max}]$ and the denominator is bounded away from zero. Thus there is a constant $L = L(I, \sigma^2, \lambda_{\max})$ such that
\[
    |\alpha_i(t) - \alpha_i(s)| \;\le\; L\,|t - s| \qquad \text{for all } i \text{ and } s, t \in I.
\]
Therefore,
\[
    \|v^*(z,t) - v^*(z, s)\|^2
    \;=\; \sum_{i=1}^d \bigl(\alpha_i(t) - \alpha_i(s)\bigr)^2 z_i^2
    \;\le\; L^2\,\|z\|^2\,(t - s)^2.
\]
Taking $s = \hat t(z)$ yields
\begin{equation}\label{eq:gap_to_time_est}
    \Delta \;\le\; L^2\,\mathbb{E}\!\left[\|z\|^2\,(\hat t(z) - t)^2\right].
\end{equation}

Since $I$ is closed and $t^* = 1/(1+\sigma^2) \notin I$, $|q'(t)|$ is bounded below on $I$, so $q^{-1}$ is Lipschitz on $q(I)$. Apply this with $\hat t(z)$ from~\eqref{eq:t_hat_app}: since projection onto $q(I)$ cannot increase distance to $q(t) \in q(I)$, there is a constant $C = C(I, \sigma^2)$ such that
\[
    |\hat t(z) - t| \;\le\; C\,|\widehat q(z) - q(t)|.
\]
Therefore
\begin{equation}\label{eq:time_to_var_est}
    \mathbb{E}\!\left[\|z\|^2(\hat t(z) - t)^2\right] \;\le\; C^2\,\mathbb{E}\!\left[\|z\|^2(\widehat q(z) - q(t))^2\right].
\end{equation}

\textbf{Step 2: Bound the weighted variance-estimation error.}
Condition on $t$. In the noise subspace,
\[
    P_\perp z_t \;\sim\; \mathcal{N}(0, q(t) I_m),
\]
so
\[
    \widehat q \;=\; q(t)\,\frac{Y}{m},
    \qquad
    Y \sim \chi^2_m.
\]
For readability, write $q := q(t)$, so
\[
    \widehat q - q \;=\; \frac{q}{m}(Y - m).
\]

Split the energy into signal and noise components:
\[
    \|z\|^2 \;=\; \|P_x z\|^2 + \|P_\perp z\|^2.
\]
We bound the two resulting terms separately.

For the signal component, $\widehat q$ depends only on $P_\perp z$, and under the diagonal spiked Gaussian model with independent $\varepsilon, x$, the projected blocks $P_x z$ and $P_\perp z$ are independent conditional on $t$. Hence
\[
    \mathbb{E}\!\left[\|P_x z\|^2 (\widehat q - q)^2 \mid t\right]
    \;=\; \mathbb{E}[\|P_x z\|^2 \mid t]\cdot \mathbb{E}[(\widehat q - q)^2 \mid t].
\]
The first factor is $O(k)$, since the $k$ signal coordinates have uniformly bounded variance on $I$. The second factor is
\[
    \mathbb{E}[(\widehat q - q)^2 \mid t]
    \;=\; q^2 \,\mathrm{Var}(Y/m)
    \;=\; \frac{2 q^2}{m}
    \;=\; O\!\left(\frac{1}{m}\right),
\]
because $q(t)$ is bounded on $I$. Therefore,
\[
    \mathbb{E}\!\left[\|P_x z\|^2(\widehat q - q)^2 \mid t\right] \;=\; O\!\left(\frac{k}{m}\right).
\]

For the noise component, use
\[
    \|P_\perp z\|^2 \;=\; q\,Y,
    \qquad
    \widehat q - q \;=\; \frac{q}{m}(Y - m).
\]
Then
\[
    \mathbb{E}\!\left[\|P_\perp z\|^2(\widehat q - q)^2 \mid t\right] \;=\; \frac{q^3}{m^2}\,\mathbb{E}[Y(Y-m)^2].
\]
By standard chi-square moment bounds,\footnote{Indeed, for $Y \sim \chi^2_m$, $\mathbb{E}[(Y-m)^2] = 2m$ and $\mathbb{E}[(Y-m)^3] = 8m$, so $\mathbb{E}[Y(Y-m)^2] = 2m^2 + 8m$.}
\[
    \mathbb{E}[Y(Y-m)^2] \;=\; O(m^2).
\]
Since $q(t)$ is uniformly bounded on $I$, the noise contribution is
\[
    \mathbb{E}\!\left[\|P_\perp z\|^2(\widehat q - q)^2 \mid t\right] \;=\; O(1).
\]

Combining the signal and noise contributions,
\[
    \mathbb{E}\!\left[\|z\|^2(\widehat q - q)^2 \mid t\right]
    \;=\; O\!\left(\frac{k}{m}\right) + O(1)
    \;=\; O\!\left(1 + \frac{k}{m}\right)
    \;=\; O\!\left(\frac{d}{d-k}\right).
\]
The bound is uniform over $t \in I$, so averaging over $t \sim \mathrm{Unif}(I)$ gives
\[
    \mathbb{E}\!\left[\|z\|^2(\widehat q(z) - q(t))^2\right] \;=\; O\!\left(\frac{d}{d-k}\right).
\]
Combining this with~\eqref{eq:gap_to_time_est} and~\eqref{eq:time_to_var_est} proves $\Delta = O(d/(d-k))$.
\end{proof}

\subsection{Independent Coupling Has Extensive Variance}\label{app:independent_coupling_variance}

The coupling variance $\mathcal{C}_\pi$ is the irreducible regression floor of (time-conditioned) flow matching: it is paid by even the optimal time-aware predictor and depends only on the choice of coupling $\pi$ (Lemma~\ref{lem:fm_bias_variance}). Under independent coupling and the baseline spiked Gaussian model (Subsection~\ref{app:spiked_model_notation}), $\mathcal{C}_\pi$ scales linearly with the ambient dimension; the proof reduces to a per-coordinate Gaussian closed form.

\begin{aproposition}[Independent coupling has extensive variance]\label{prop:independent_coupling_variance}
Under the spiked Gaussian model with independent coupling and $t \sim \mathrm{Unif}(I)$,
\[
    \mathcal{C}_\pi \;=\; \Omega(d).
\]
\end{aproposition}

\begin{proof}
\textbf{Coordinate reduction.}
Recall that, working in the eigenbasis of \(\Sigma\), we write
\[
    x_i \sim \mathcal N(0,S_i),
\]
where \(S_i\) denotes the \(i\)-th eigenvalue of the data covariance \(\Sigma\). Under the spiked model,
\[
    S_i =
    \begin{cases}
        \lambda_i+\sigma^2, & i \le k,\\
        \sigma^2, & i > k.
    \end{cases}
\]
Because the independent coupling makes \(x\) and \(\varepsilon\) independent, and because both are diagonal in this basis, the coordinates decouple. Hence it suffices to compute the one-dimensional conditional variance and then sum over coordinates.

\textbf{Gaussian conditioning.}
Fix a coordinate \(i\). Conditional on \(t\), the pair
\[
    u_i = x_i-\varepsilon_i,
    \qquad
    z_i = (1-t)\varepsilon_i + t x_i
\]
is jointly Gaussian. Since \(x_i \sim \mathcal N(0,S_i)\) and
\(\varepsilon_i \sim \mathcal N(0,1)\) are independent,
\[
    \operatorname{Var}(u_i \mid t) = S_i+1,
\]
\[
    \operatorname{Var}(z_i \mid t)
    = (1-t)^2 + t^2 S_i,
\]
and
\[
    \operatorname{Cov}(u_i,z_i \mid t)
    =
    \operatorname{Cov}(x_i-\varepsilon_i,
    (1-t)\varepsilon_i+t x_i)
    =
    tS_i-(1-t).
\]

Since \((u_i,z_i)\) is jointly Gaussian conditional on \(t\), the scalar Gaussian conditioning formula gives
\[
    \operatorname{Var}(u_i\mid z_i,t)
    =
    \operatorname{Var}(u_i\mid t)
    -
    \frac{\operatorname{Cov}(u_i,z_i\mid t)^2}
    {\operatorname{Var}(z_i\mid t)}.
\]
Substituting the above expressions gives
\[
    \operatorname{Var}(u_i\mid z_i,t)
    =
    (S_i+1)
    -
    \frac{(tS_i-(1-t))^2}{(1-t)^2+t^2S_i}.
\]

A direct simplification yields
\[
    \operatorname{Var}(u_i \mid z_i,t)
    =
    \frac{S_i}{(1-t)^2+t^2S_i}.
\]

\textbf{Uniform lower bound and summation.}
Since \(S \mapsto S/((1-t)^2+t^2S)\) is increasing for each fixed \(t\in I\), and since \(S_i\ge \sigma^2>0\), we have
\[
    \frac{S_i}{(1-t)^2+t^2S_i}
    \ge
    \frac{\sigma^2}{(1-t)^2+t^2\sigma^2}.
\]

Because \(I\subset(0,1)\) is closed and bounded, the right-hand side has a positive minimum on \(I\). Define
\[
    c_0
    :=
    \min_{t\in I}
    \frac{\sigma^2}{(1-t)^2+t^2\sigma^2}
    >0.
\]
Therefore, uniformly over all coordinates \(i\) and all \(t\in I\),
\[
    \operatorname{Var}(u_i\mid z_i,t)
    \ge c_0.
\]
Thus,
\[
    \mathcal C_\pi
    =
    \sum_{i=1}^d
    \mathbb E_t\!\left[
        \operatorname{Var}(u_i\mid z_i,t)
    \right]
    \ge
    \sum_{i=1}^d c_0
    =
    c_0 d.
\]
Hence \(\mathcal C_\pi=\Omega(d)\).

\end{proof}

\subsection{Relative Cost of Time-Independence}\label{app:relative_time_independence_cost}

Combining the two preceding propositions: the time-blindness gap $\Delta$ is bounded relative to the coupling variance $\mathcal{C}_\pi$ by $O(1/(d-k))$. This is the headline result, equivalent to Theorem~\ref{thm:time_independence_cost} in the main text.

\begin{atheorem}[Relative cost of time-independence]\label{thm:relative_time_independence_cost}
Under the spiked Gaussian model with independent coupling and $t \sim \mathrm{Unif}(I)$,
\begin{equation}\label{eq:delta_ratio_bound}
    \frac{\Delta}{\mathcal{C}_\pi}
    \;=\; O\!\left(\frac{1}{d-k}\right).
\end{equation}
In particular, holding $k$ and the spectral parameters fixed, the ratio is $O(1/d)$.
\end{atheorem}

\begin{proof}
By Proposition~\ref{prop:absolute_time_blindness_gap}, there is a constant $C_1 = C_1(I, \sigma^2, \lambda_{\max})$ such that
\[
    \Delta \;\le\; C_1\,\frac{d}{d-k}.
\]
By Proposition~\ref{prop:independent_coupling_variance}, $\mathcal{C}_\pi \ge c_0\, d$ for a constant $c_0 = c_0(I, \sigma^2, \lambda_{\max}) > 0$. Dividing,
\[
    \frac{\Delta}{\mathcal{C}_\pi}
    \;\le\; \frac{C_1\, d / (d-k)}{c_0\, d}
    \;=\; \frac{C_1/c_0}{d-k}
    \;=\; O\!\left(\frac{1}{d-k}\right),
\]
where the $d$ factor in the numerator and denominator cancels, leaving the rate determined entirely by the noise-subspace dimension $d-k$.
\end{proof}

\noindent The bound is constructive: Proposition~\ref{prop:absolute_time_blindness_gap} upper-bounds the optimal time-blind predictor by exhibiting one specific $g^*$ that uses the noise-subspace estimator; the actual optimum $\bar{v}(z) = \mathbb{E}[u \mid z]$ can only do better. The ambient dimension $d$ enters only through the energy normalization $\mathbb{E}[\|z\|^2] = \Theta(d)$ on both sides of the ratio and cancels; the statistical content is the $O(1/(d-k))$ rate, controlled by the noise-subspace dimension alone. Holding $k$ fixed and varying $d$ gives the $O(1/d)$ ratio scaling; holding $d - k$ fixed and varying $k$ leaves the ratio bounded by a constant, the prediction that pulls the bound apart from the per-dimension account of \citet{sun2025noiseconditioning} (Section~\ref{sec:related_work}). Note that the absolute gap $\Delta = O(d/(d-k))$ does not itself vanish when $d - k$ is comparable to $d$: it remains $O(1)$, and what vanishes is the ratio. Methods that suppress the coupling variance directly (optimal transport, mini-batch OT, rectified flows; Section~\ref{sec:cost_time_independence}) drive the dominant Term~II down, leaving $\Delta$ as the only remaining residual.

\section{Experimental Details}\label{app:experiments}

\subsection{Datasets and Architectures}\label{app:datasets}

Our experiments cover three datasets with two distinct backbones. CIFAR-10 is modeled in pixel space with an ADM UNet; CelebA-HQ and FFHQ are modeled in the latent space of a pretrained VAE using a SiT transformer. We document each setup below; later experiment subsections reference the relevant setup rather than redefining it.

\paragraph{Backbones.}
For CIFAR-10 we use an ADM UNet~\citep{dhariwal2021diffusion} with $128$ base channels, channel multipliers $[1, 2, 2, 2]$, $2$ residual blocks per resolution, multi-head attention at $16 \times 16$ resolution ($4$ heads, $64$ channels per head), and dropout $0.1$, totaling approximately $35.7$M parameters. For CelebA-HQ and FFHQ we use SiT-B/2: a transformer of depth $12$, hidden size $768$, $12$ attention heads, and patch size $2$, acting on $16 \times 16 = 256$ tokens after encoding $256 \times 256$ images to $4 \times 32 \times 32$ latents with the Stable Diffusion 1.5 VAE (\texttt{stabilityai/sd-vae-ft-ema}, latents cached once before training).

\begin{table}[t]
\centering
\small
\caption{Per-dataset training hyperparameters.}
\label{tab:hyperparams}
\begin{tabular}{lccc}
\toprule
\textbf{Setting} & \textbf{CIFAR-10} & \textbf{CelebA-HQ} & \textbf{FFHQ} \\
\midrule
\multicolumn{4}{l}{\textit{Architecture}} \\
Backbone & ADM UNet ($\approx\!35.7$M) & SiT-B/2 & SiT-B/2 \\
Image space & pixel & VAE latent (SD 1.5) & VAE latent (SD 1.5) \\
Resolution / $d$ & $32 \times 32 \times 3$, \ $3{,}072$ & $4 \times 32 \times 32$, \ $4{,}096$ & $4 \times 32 \times 32$, \ $4{,}096$ \\
\midrule
\multicolumn{4}{l}{\textit{Optimization}} \\
Optimizer & AdamW & AdamW & AdamW \\
Learning rate & $2 \times 10^{-4}$ & $1 \times 10^{-4}$ & $1 \times 10^{-4}$ \\
Gradient clip & $1.0$ & $1.0$ & $1.0$ \\
Warmup steps & $5{,}000$ & $5{,}000$ & $5{,}000$ \\
EMA decay & $0.9999$ & $0.9999$ & $0.9999$ \\
Magnitude function & \texttt{truncated}, $a = 0.8$ & --- & --- \\
\midrule
\multicolumn{4}{l}{\textit{Training}} \\
Iterations & $200{,}000$ & $100{,}000$ & $100{,}000$ \\
Batch size & $128$ per GPU & $256$ global & $256$ global \\
Hardware & 1$\times$ H200 & 1$\times$ H200 & 1$\times$ H200 \\
\midrule
\multicolumn{4}{l}{\textit{Sampling}} \\
Sampler & NAG & Euler & Euler \\
\bottomrule
\end{tabular}
\end{table}

\subsection{Loss Decomposition Validation}\label{app:exp_loss_decomp}

\paragraph{Setup.}
We train four flow matching models per dataset in a $2 \times 2$ factorial design, varying time conditioning (on/off) and coupling (independent/OT), on CIFAR-10, CelebA-HQ, and FFHQ. Within each dataset, all four variants share the identical architecture and hyperparameters; the only differences are (i)~whether the coupling is independent or mini-batch OT (exact EMD \texttt{pot\_emd}) and (ii)~whether the model conditions on $t$ or receives $t = 0$ at every step. The ``time-blind'' variant zeros the time input fed to the network at every step while keeping the time-embedding layers in the architecture, so the parameter count is identical across all four models. Per-dataset architectures, preprocessing, and training hyperparameters are specified in Appendix~\ref{app:datasets}; we summarize the cells below.

\paragraph{CIFAR-10.}
Pixel-space ADM UNet (Table~\ref{tab:hyperparams}) trained for $200{,}000$ steps with per-GPU batch size $128$. OT coupling uses exact EMD with $B = 128$.

\paragraph{CelebA-HQ.}
SiT-B/2 in the Stable Diffusion VAE latent space (Table~\ref{tab:hyperparams}) trained for $100{,}000$ steps with global batch size $256$. OT coupling uses exact EMD with $B = 256$.

\paragraph{FFHQ.}
SiT-B/2 in the Stable Diffusion VAE latent space (Table~\ref{tab:hyperparams}), with the same protocol as CelebA-HQ: $100{,}000$ steps, global batch size $256$, exact EMD with $B = 256$.

\paragraph{MSE evaluation protocol.}
We compute the mean MSE under EMA weights in eval mode as a measurement of the converged training-loss floor of each cell. The measured quantity is the irreducible loss floor of the training distribution: a property of the joint distribution over $(x, \varepsilon, t)$, not a generalization claim. To ensure a fair comparison, each model is evaluated under the coupling it was trained with: OT models (A, C) are evaluated with the same \texttt{pot\_emd} permutation applied per eval batch (using the $B$ matching training), while independent models (B, D) use random pairing.

\emph{CIFAR-10.}
CIFAR-10 has an official train/test partition (inherited from its classification-dataset origin), so we evaluate on its $10{,}000$-image test split (78 batches of 128 samples). The train-vs-test loss gap is $\le 0.0015$, well within the per-batch standard error, so this distinction does not affect the decomposition.

\emph{CelebA-HQ and FFHQ.}
These benchmarks are introduced in the literature as generative-model benchmarks without a canonical train/test partition~\citep{karras2018progressivegrowinggansimproved,karras2019stylebasedgeneratorarchitecturegenerative}, so following the convention of those benchmarks we evaluate on the same VAE-latent image set used for training, with fresh Gaussian noise $\varepsilon \sim \mathcal{N}(0, I)$ and uniform $t \sim \mathcal{U}[0,1]$ drawn for each batch (78 batches of 256 samples).

\paragraph{FID evaluation protocol.}
For each of the four cells per dataset we generate $50{,}000$ samples by integrating the trained vector field with $250$ Euler steps from $t=0$ to $t=1$ using the EMA weights, and report clean-FID~\citep{parmar2022aliased}. All four cells (A, B, C, D) per dataset use the identical sampler, sample count, and reference set; the only differences across cells are the model weights themselves. Per-dataset scoring details follow.

\emph{CIFAR-10.}
Samples are scored directly in pixel space against the standard CIFAR-10 training-set reference statistics.

\emph{CelebA-HQ and FFHQ.}
Samples are produced in VAE-latent space and decoded back to $256 \times 256$ pixels with the same \texttt{stabilityai/sd-vae-ft-ema} VAE used at training time before scoring against each dataset's full real distribution as the reference, following the convention of those benchmarks~\citep{karras2018progressivegrowinggansimproved,karras2019stylebasedgeneratorarchitecturegenerative}.

\begin{table}[t]
\centering
\caption{\textbf{Converged loss floor (MSE) and FID across datasets and methods.} MSE is the converged training-loss floor evaluated as described in the MSE protocol above (CIFAR-10 on its official test split; CelebA-HQ and FFHQ on the same image set with fresh $\varepsilon, t$, since no canonical test partition exists for these benchmarks). FID is clean-FID-50k against each dataset's full real distribution.}
\label{tab:main_results}
\begin{tabular}{l cc cc cc}
\toprule
 & \multicolumn{2}{c}{CIFAR-10} & \multicolumn{2}{c}{CelebA-HQ} & \multicolumn{2}{c}{FFHQ} \\
\cmidrule(lr){2-3} \cmidrule(lr){4-5} \cmidrule(lr){6-7}
Method & MSE $\downarrow$ & FID $\downarrow$ & MSE $\downarrow$ & FID $\downarrow$ & MSE $\downarrow$ & FID $\downarrow$ \\
\midrule
Time-conditioned, naive coupling (\textcolor{blindorange}{$\bullet$}) & 0.170 & 5.40 & 0.597 & 11.35 & 0.671 & 30.15 \\
Time-conditioned, OT coupling (\textcolor{condblue}{$\bullet$}) & 0.145 & 5.35 & 0.559 & 11.10 & 0.634 & 29.34 \\
Time-blind, naive coupling (\textcolor{blindorange}{$\circ$}) & 0.171 & 5.31 & 0.611 & 9.93 & 0.675 & 29.35 \\
Time-blind, OT coupling (\textcolor{condblue}{$\circ$}) & 0.146 & 5.22 & 0.567 & 10.31 & 0.637 & 29.10 \\
\bottomrule
\end{tabular}
\end{table}

\paragraph{Results.}
Converged loss-floor values for all three datasets are reported in Table~\ref{tab:main_results}. On CIFAR-10, the coupling effect is $\delta_{\text{coupling}} = 0.025$ regardless of time conditioning, and the time-blindness effect is $\delta_{\text{time}} = 0.001$. On CelebA-HQ, the coupling effect is $\delta_{\text{coupling}} = 0.038$--$0.044$ and the time-blindness effect is $\delta_{\text{time}} = 0.008$--$0.014$. On FFHQ, the coupling effect is $\delta_{\text{coupling}} = 0.037$--$0.038$ and the time-blindness effect is $\delta_{\text{time}} = 0.003$--$0.004$. In all cases the coupling variance dominates and the decomposition is approximately additive.

\subsection{Empirical Validation of Time Identifiability}\label{app:exp_identifiability}

We validate the theoretical predictions of Theorem~\ref{thm:t_identifiable} in four settings: (i)~a fully synthetic spiked covariance model where all parameters are controlled exactly, and (ii--iv)~real datasets (CIFAR-10 in pixel space, $d=3072$; FFHQ and CelebA-HQ in VAE latent space, $d=4096$), where the spiked covariance is fitted via PCA and used only for the theoretical prediction, not for data generation.

\subsubsection{Common Estimator}

In both settings, we observe interpolants of the form $z_i = (1 - t_i)\,\varepsilon_i + t_i\,x_i$, where $\varepsilon_i \sim \mathcal{N}(0, I_d)$ is the source noise and $x_i$ is drawn from a distribution with covariance $\Sigma = U_x\Lambda U_x^\top + \sigma^2 I_d$. The interpolation times are sampled as $t_i \sim \mathrm{Uniform}[0, 1]$ over the full path; we report the mean absolute error $\mathrm{MAE} := \frac{1}{N}\sum_i |\hat t_i - t_i|$, which absorbs both the in-regime ($t < t^*$) error and any branch-mismatch contribution near the right endpoint.

\paragraph{Estimator.}
Given $z_i$ alone, we estimate $t_i$ via the noise-subspace statistic.
Let $U_\perp \in \mathbb{R}^{d \times (d-k)}$ span the orthogonal complement of the signal subspace spanned by $U_x$, let $P_\perp := U_\perp U_\perp^\top$ denote the orthogonal projector onto the noise subspace, and define
\[
  \hat{\sigma}_\perp^2
  \;:=\;
  \frac{1}{d-k}\,\|P_\perp z_i\|^2.
\]
In the noise subspace, the interpolant has variance $\sigma_\perp^2(t_i) := (1 - t_i)^2 + t_i^2\,\sigma^2$, so $\hat{\sigma}_\perp^2$ concentrates around this deterministic function of $t_i$.

\subsubsection{Discarded Samples}\label{app:discarded_samples}

We compute $\hat{t}_i$ via the closed-form inversion~\eqref{eq:t_hat_inversion} derived in Appendix~\ref{app:closed_form_inversion}. The descending-branch minus sign applies in all four settings: Table~\ref{tab:dataset_sigma2} reports per-dataset $\sigma_\perp^2$ and $t^*$, with $t^* \ge 0.89$ throughout, so the descending branch covers nearly the full $[0,1]$ sampling range. The only finite-sample wrinkle is the discriminant, which can become negative. A short rearrangement of the discriminant in~\eqref{eq:t_hat_two_roots} gives
\[
  4 - 4(1+\sigma^2)(1 - \hat\sigma_\perp^2)
  \;=\;
  4\bigl[(1+\sigma^2)\,\hat\sigma_\perp^2 - \sigma^2\bigr],
\]
which is negative precisely when $\hat\sigma_\perp^2 < \sigma^2/(1+\sigma^2) = \sigma_\perp^2(t^*)$, i.e.\ when the empirical noise-subspace variance falls below the global minimum of the clock and therefore lies outside the range of $\sigma_\perp^2$. This can only occur as a finite-sample fluctuation: the chi-square statistic in~\eqref{eq:var_clock_statistic} occasionally dips below its mean far enough to push the estimate into the unattainable region. We discard such samples; in practice fewer than $0.1\%$ are affected.

\subsubsection{Theoretical Prediction}\label{app:theoretical_prediction}

The proof of Theorem~\ref{thm:t_identifiable} establishes the rate $\hat t - t = O_p(1/\sqrt{d-k})$ via Chebyshev's inequality, but for direct comparison with experiments we want a quantitative reference distribution rather than just a rate. We derive one below in three short moves: the exact variance of the noise-subspace statistic, its propagation through the inverse map, and the aggregation across sampled times.

\paragraph{Variance of the clock statistic.}
Conditioned on $t_i$, the components of $U_\perp^\top z_i$ are i.i.d.\ $\mathcal{N}(0, \sigma_\perp^2(t_i))$, so $(d-k)\,\hat\sigma_\perp^2 / \sigma_\perp^2(t_i) \sim \chi^2_{d-k}$. The variance of the chi-square distribution then gives the exact conditional variance
\begin{equation}\label{eq:var_clock_statistic}
    \mathrm{Var}\!\bigl(\hat\sigma_\perp^2 \mid t_i\bigr)
    \;=\; \sigma_\perp^4(t_i)\cdot\frac{2}{d-k}.
\end{equation}
For large $d-k$, the central limit theorem additionally implies $\hat\sigma_\perp^2$ is approximately Gaussian. Equation~\eqref{eq:var_clock_statistic} is the only stochastic input we need; everything that follows propagates it deterministically through the inverse map.

\paragraph{Delta-method propagation to time.}
Let $g := (\sigma_\perp^2)^{-1}$ denote the inverse map sending $\hat\sigma_\perp^2$ to $\hat t_i$. Its derivative at the true variance is
\[
    g'\!\bigl(\sigma_\perp^2(t_i)\bigr) \;=\; \frac{1}{\dot\sigma_\perp^2(t_i)},
    \qquad
    \dot\sigma_\perp^2(t) \;=\; -2(1-t) + 2t\sigma^2.
\]
A first-order Taylor expansion of $g$ around $\sigma_\perp^2(t_i)$ gives the standard delta-method approximation
\[
    \hat t_i - t_i \;\approx\; g'\!\bigl(\sigma_\perp^2(t_i)\bigr)\cdot\bigl(\hat\sigma_\perp^2 - \sigma_\perp^2(t_i)\bigr),
\]
and squaring then taking variances yields
\begin{equation}\label{eq:var_t_hat_delta}
    \mathrm{Var}\!\bigl(\hat t_i - t_i \mid t_i\bigr)
    \;\approx\;
    \underbrace{\frac{1}{\bigl[-2(1-t_i) + 2 t_i \sigma^2\bigr]^2}}_{\bigl[g'(\sigma_\perp^2(t_i))\bigr]^2}
    \cdot
    \underbrace{\sigma_\perp^4(t_i)\cdot\frac{2}{d-k}}_{\mathrm{Var}(\hat\sigma_\perp^2 \mid t_i)}.
\end{equation}
Because $g$ is a smooth map applied to a CLT-regime statistic, $\hat t_i - t_i$ inherits the Gaussian approximation of $\hat\sigma_\perp^2$ to leading order. This is the per-sample reference distribution we will compare against the empirical estimator.

\paragraph{Dataset-level prediction.}
To produce a single reference curve per dataset, we average~\eqref{eq:var_t_hat_delta} over the $N$ sampled times $\{t_i\}_{i=1}^N$, excluding samples within $0.05$ of the critical point $t_{\mathrm{crit}} = 1/(1+\sigma^2)$ (where $g'$ nearly vanishes and the first-order Taylor expansion breaks down). This yields the dataset-level standard deviation
\begin{equation}\label{eq:sigma_theory_avg}
    \bar\sigma_{\mathrm{theory}} \;=\; \sqrt{\frac{1}{N}\sum_{i=1}^N \mathrm{Var}\!\bigl(\hat t_i - t_i \mid t_i\bigr)},
\end{equation}
and we plot $\mathcal{N}(0, \bar\sigma_{\mathrm{theory}}^2)$ against the empirical histogram of $\hat t_i - t_i$. Equation~\eqref{eq:sigma_theory_avg} is the prediction tested in Figure~\ref{fig:t_estimation}.

\subsubsection{Setting 1: Synthetic Spiked Covariance}

We sample $N = 10{,}000$ independent triples $(\varepsilon_i, x_i, t_i)$ with
$\varepsilon_i \sim \mathcal{N}(0, I_d)$,
$x_i \sim \mathcal{N}(0, U_x\Lambda U_x^\top + \sigma^2 I_d)$, and
$t_i \sim \mathrm{Uniform}[0, 1]$.
Here $U_x \in \mathbb{R}^{d \times k}$ is a random orthonormal basis (obtained via QR decomposition of a Gaussian random matrix),
$\Lambda = \mathrm{diag}(\lambda_1, \dots, \lambda_k)$ with eigenvalues spaced linearly between $1.0$ and $10.0$, and $\sigma^2 = 0.1$.

To verify the $1/\sqrt{d-k}$ scaling, we fix $k = 10$ and sweep the ambient dimension across $d \in \{20, 50, 100, 200, 500, 1000\}$, giving noise subspace dimensions $d - k \in \{10, 40, 90, 190, 490, 990\}$. For each $d$, we compute $\hat{t}_i$ for all $N$ samples and record the empirical standard deviation of $\hat{t}_i - t_i$.

\paragraph{Results (Figure~\ref{fig:t_estimation}, panels a--b).}
Panel~(a) shows the empirical standard deviation versus $d - k$ alongside the theoretical prediction; the two curves are nearly indistinguishable, confirming that the delta method approximation is tight. Panel~(b) shows the full distribution of estimation errors for the setting $k = 10$, $d - k = 990$: the empirical KDE closely tracks the theoretical Gaussian $\mathcal{N}(0, \bar{\sigma}^2_{\mathrm{theory}})$, with the empirical distribution being slightly more concentrated than the theoretical prediction. This is expected, as the theoretical variance is a first-order upper approximation and the actual estimator benefits from higher-order concentration.

\subsubsection{Settings 2--4: Real Datasets with Fitted Covariance Prediction}\label{app:cifar10_pca}

The synthetic setting validates the theory in the exact model for which it was derived. We now ask a stronger question: does the Gaussian-derived theoretical prediction remain accurate when applied to \emph{real, non-Gaussian data}? To test this, we use actual images (or VAE latents) as the data distribution $x_i$ and only use the Gaussian approximation for the theoretical variance prediction, not for data generation. We evaluate on CIFAR-10 (pixel space, $d=3072$), FFHQ (VAE latents, $d=4096$), and CelebA-HQ (VAE latents, $d=4096$).

\paragraph{Common procedure.}
For each dataset we compute the sample mean $\mu$ and sample covariance $\hat{\Sigma} = \frac{1}{n-1}\sum_i (x_i - \mu)(x_i - \mu)^\top$ over the full set of $n$ available samples (CIFAR-10 in pixel space, CelebA-HQ and FFHQ in VAE-latent space), and form the eigendecomposition $\hat{\Sigma} = V\,\mathrm{diag}(\hat{\lambda}_1, \dots, \hat{\lambda}_d)\,V^\top$ with eigenvalues sorted in decreasing order. Given a signal rank $k$, we approximate the data covariance with a spiked model by setting the signal subspace $U_x = V_{:, 1:k}$, the spike eigenvalues $\Lambda = \mathrm{diag}(\hat{\lambda}_1 - \sigma^2, \dots, \hat{\lambda}_k - \sigma^2)$ (clipped to be non-negative), and the isotropic noise variance $\sigma^2 = \frac{1}{d-k}\sum_{j=k+1}^{d} \hat{\lambda}_j$. These parameters are used solely for the theoretical variance prediction; the data itself is not assumed to be Gaussian. We then sample $N = 20{,}000$ triples $(\varepsilon_i, x_i, t_i)$ with $\varepsilon_i \sim \mathcal{N}(0, I_d)$, $x_i$ drawn with replacement from the dataset (and centered by subtracting $\mu$), and $t_i \sim \mathrm{Uniform}[0, 1]$, form the interpolant $z_i = (1 - t_i)\,\varepsilon_i + t_i\,x_i^{\mathrm{centered}}$, and apply the noise-subspace estimator using the PCA-derived $U_\perp$.

The representative signal rank $k$ for the per-dataset error distribution plots (Figure~\ref{fig:t_estimation}) is chosen at the $95\%$ cumulative-explained-variance threshold, which falls in the regime where the spectrum transitions from rapid to gradual decay (Figure~\ref{fig:explained_variance}). The scaling sweep additionally varies $k \in \{10, 50, 100, 200, 500, 1000, 2000\}$ within each dataset to trace the full range of noise-subspace dimensions, with $\sigma^2$ recomputed from the residual spectrum at each $k$.

\paragraph{CIFAR-10.}
We load all $50{,}000$ training images and flatten each $32 \times 32 \times 3$ image to a $d = 3{,}072$-dimensional vector (pixel values in $[0, 1]$). The $95\%$ explained-variance threshold corresponds to $k = 217$, which is the rank used for the error-distribution panel; the mean absolute error is $0.0062$.

\paragraph{FFHQ.}
We use VAE latents at $d = 4{,}096$ (Stable Diffusion VAE, $4 \times 32 \times 32$); see Table~\ref{tab:hyperparams} for the encoding setup. The $95\%$ explained-variance threshold corresponds to $k = 2{,}581$; the mean absolute error is $0.0213$.

\paragraph{CelebA-HQ.}
We use VAE latents at $d = 4{,}096$ from the same Stable Diffusion VAE used for FFHQ; see Table~\ref{tab:hyperparams}. The $95\%$ explained-variance threshold corresponds to $k = 2{,}273$; the mean absolute error is $0.0175$.

\paragraph{Results (Figure~\ref{fig:t_estimation}).}
Figure~\ref{fig:t_estimation} shows the error distribution at the 95\% explained-variance rank for all four settings. Despite using real, non-Gaussian data, the empirical distributions closely match the Gaussian-derived theoretical prediction in every case, confirming that the bound is tight beyond the Gaussian model. The scaling sweep (not shown in the main figure) further confirms the $1/\sqrt{d-k}$ rate across all datasets. Notably, the VAE latent distributions (FFHQ, CelebA-HQ) have much flatter spectra than CIFAR-10 pixels: 95\% explained variance requires $k \approx 2{,}581$ and $k \approx 2{,}273$ respectively, compared to $k = 217$ for CIFAR-10 (Figure~\ref{fig:explained_variance}). This means the VAE spreads variance across more components, leaving a smaller effective noise subspace. Despite this, the estimator remains accurate in all cases, likely because the projection onto a high-dimensional noise subspace has a CLT-like averaging effect regardless of the marginal distribution of $x_i$.

\subsection{Time-Localized Estimation Error and the Variance-Clock Critical Point}\label{app:exp_clock_critical_point}

Theorem~\ref{thm:t_identifiable} requires staying away from the critical point $t^* = 1/(1 + \sigma^2)$ at which the scalar residual-variance clock $\sigma_\perp^2(t) = (1-t)^2 + t^2\sigma^2$ has zero derivative. This subsection probes how this caveat manifests on real datasets by (a) locating $t^*$ from the fitted residual variance per dataset, and (b) sweeping the noise-subspace estimator across $t \in [0, 1]$ and reporting binned mean absolute error.

\paragraph{Setup.}
For each of CIFAR-10, CelebA-HQ, and FFHQ we run PCA on the data, pick $k$ at the $95\%$ cumulative-explained-variance threshold, and fit $\hat\sigma_{\perp,\mathrm{data}}^2 = \mathrm{mean}(\hat\lambda_{j>k})$ from the trailing eigenvalues, giving the predicted clock and its critical point. We then sample $t \sim \mathrm{Uniform}[0, 1]$ (with extra mass at $t = 1$ to populate the right edge), draw a real datapoint $x$ and noise $\varepsilon \sim \mathcal{N}(0, I_d)$, form the interpolant $z = (1-t)\varepsilon + tx$, and run the noise-subspace estimator $\hat t(z)$ from~\eqref{eq:t_hat_def} via the closed-form descending-branch inverter of Appendix~\ref{app:closed_form_inversion} (the minus-sign root of the quadratic). On CelebA-HQ and FFHQ this includes $t > t^*$, where the always-descending-branch convention returns the wrong root by construction; the consequences are quantified in the second result below. We use $10^5$ samples per dataset and report the binned conditional mean absolute error $\mathbb{E}[|\hat t - t| \mid t]$.

\begin{table}[h]
  \centering
  \begin{tabular}{lrrrrr}
  \toprule
  Dataset & $d$ & $k$ & $d-k$ & $\sigma_\perp^2$ & $t^* = 1 / (1 + \sigma_\perp^2)$ \\
  \midrule
  Synthetic & --- & 10 & --- & 0.1000 & 0.9090 \\
  CIFAR-10 & 3072 & 99 & 2973 & 0.0064 & 0.9937 \\
  CelebA-HQ & 4096 & 1571 & 2525 & 0.1079 & 0.9026 \\
  FFHQ & 4096 & 1900 & 2196 & 0.1182 & 0.8943 \\
  \bottomrule
  \end{tabular}
  \vspace{0.75em}
  \caption{Per-dataset residual-subspace variance $\sigma_\perp^2$ and induced critical point $t^*$. $\sigma_\perp^2$ is the mean of the bottom $d-k$ PCA eigenvalues of the data covariance, with $k$ chosen at the $95\%$ cumulative-variance threshold.}
  \label{tab:dataset_sigma2}
\end{table}

\paragraph{Results.}
The experiment makes two predictions, one about the geometry of the variance clock and one about the empirical failure mode of the estimator.

\textbf{The clock's critical point is dataset-determined and pushed to the data endpoint by small residual variance.} \emph{Hypothesis:} since $t^* = 1/(1 + \hat\sigma_{\perp,\mathrm{data}}^2)$ is a monotone function of the fitted residual variance, datasets with small $\hat\sigma_{\perp,\mathrm{data}}^2$ should place $t^*$ near the data endpoint $t = 1$, while datasets with flatter residual spectra (larger $\hat\sigma_{\perp,\mathrm{data}}^2$) should move $t^*$ into the interior. \emph{Validation:} the fitted critical points (Table~\ref{tab:dataset_sigma2}, Figure~\ref{fig:clock_curves}) span CIFAR-10 $t^* \approx 0.994$ down to FFHQ $t^* \approx 0.894$: the pixel-space CIFAR-10 has small residual variance and $t^*$ at the boundary, while the VAE-latent datasets have flatter residual spectra and $t^*$ in the interior.

\textbf{For $t > t^*$ the descending-branch inverter returns the reflection $2t^* - t$, producing a quantitatively predictable rise in error.} \emph{Hypothesis:} the variance clock $\sigma_\perp^2(t) = (1+\sigma^2)(t - t^*)^2 + \sigma_\perp^2(t^*)$ is a parabola symmetric around $t^*$, so for any $\hat\sigma_\perp^2 \ge \sigma_\perp^2(t^*)$ the two preimages are $t_\pm = t^* \pm \delta$ with $\delta = \sqrt{(\hat\sigma_\perp^2 - \sigma_\perp^2(t^*))/(1+\sigma^2)}$. The minus-root inverter returns $t_-$, so when the true $t > t^*$ the estimator returns the reflection $2t^* - t$ rather than $t$, giving error
\[
  |\hat t - t| \;\approx\; 2(t - t^*) \quad \text{for } t > t^*.
\]
We therefore predict three things on the binned-MAE plot: (i) error stays small for $t < t^*$, (ii) error rises with slope $2$ starting near $t = t^*$, and (iii) the right-edge peak equals $2(1 - t^*)$ — predicted to be $\approx 0.012$ for CIFAR-10, $\approx 0.194$ for CelebA-HQ, and $\approx 0.212$ for FFHQ. \emph{Validation:} on CelebA-HQ and FFHQ (Figure~\ref{fig:estimator_error_t}) the rising-error region begins near the dataset-specific $t^*$ rather than at the right endpoint, with peak magnitudes matching the dataset-dependent $2(1-t^*)$ prediction. CIFAR-10's predicted peak ($0.012$) is small enough that the branch-mismatch effect is subdominant; any observed residual rise on CIFAR-10 instead reflects model-misspecification: as $t \to 1$ the marginal of $z_t$ approaches the data marginal, which is non-Gaussian and not perfectly described by the spiked-covariance fit used to compute $\sigma_\perp^2(t)$.

\begin{figure}[t]
\centering
\includegraphics[width=\linewidth]{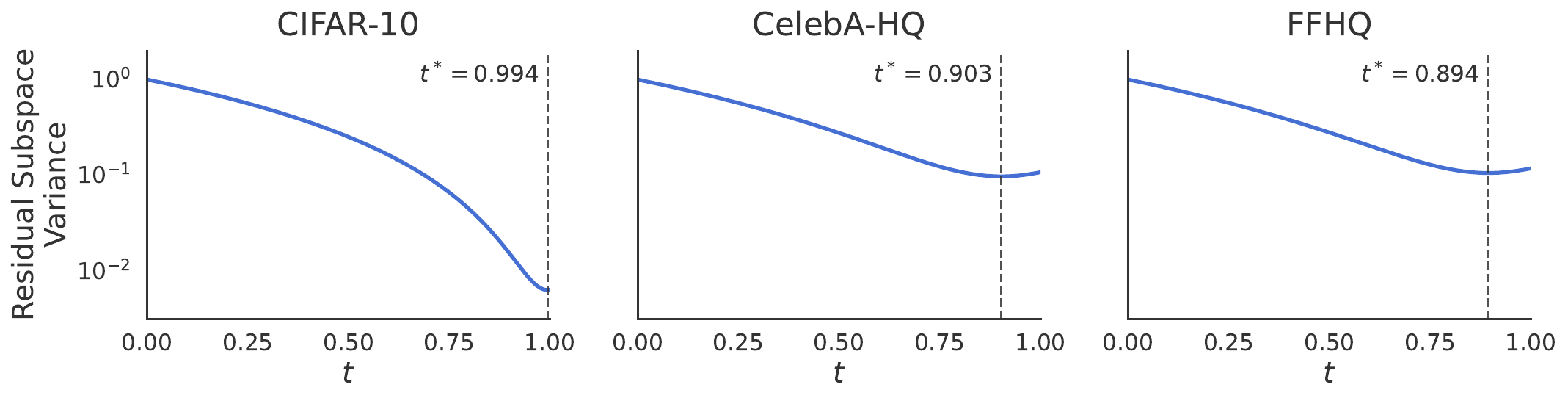}
\vspace{-1em}
\caption{\textbf{Dataset-specific residual-variance clocks.} For each dataset we estimate the residual-subspace data variance $\hat\sigma_{\perp,\mathrm{data}}^2$ after fixing the PCA signal dimension $k$, and plot $\sigma_\perp^2(t) = (1-t)^2 + t^2\,\hat\sigma_{\perp,\mathrm{data}}^2$. Dashed lines mark the predicted critical point $t^* = 1/(1 + \hat\sigma_{\perp,\mathrm{data}}^2)$, where the scalar variance statistic has zero derivative. Small residual variance (CIFAR-10) places $t^*$ near the data endpoint $t = 1$; flatter residual spectra (CelebA-HQ, FFHQ, VAE latents) move $t^*$ into the interior.}
\label{fig:clock_curves}
\end{figure}

\begin{figure}[t]
\centering
\includegraphics[width=\linewidth]{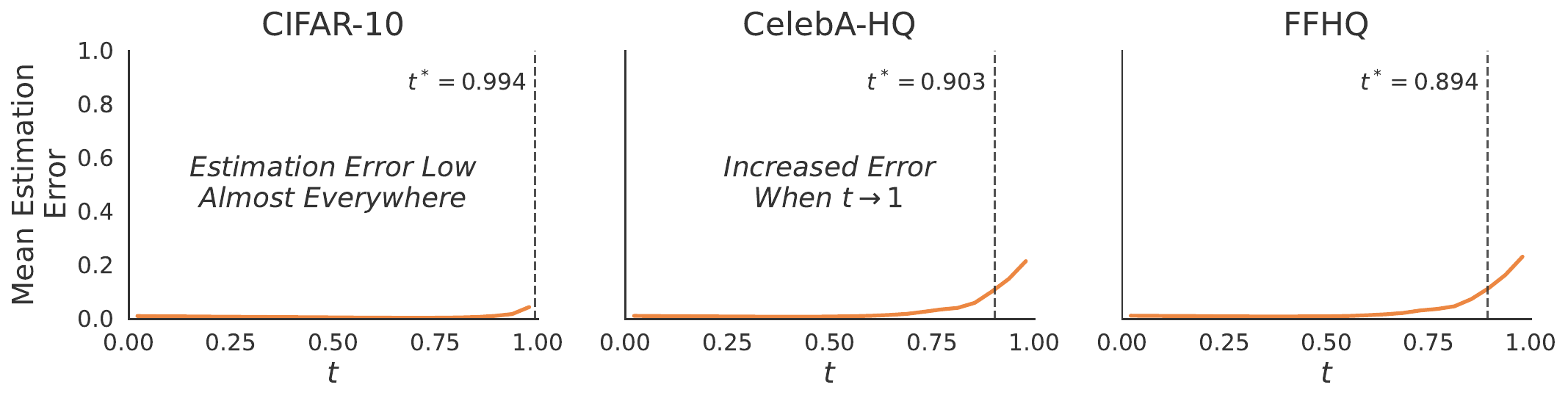}
\vspace{-1em}
\caption{\textbf{Estimator error rises for $t > t^*$, where the always-minus inverter returns the reflection $2t^* - t$ rather than $t$.} We bin samples by true interpolation time $t$ and report mean absolute error $\mathbb{E}[|\hat t - t| \mid t]$. The closed-form inverter (Appendix~\ref{app:closed_form_inversion}) always returns the descending-branch root, so for $t > t^*$ it returns $2t^* - t$ and incurs error $2(t - t^*)$. Dashed lines mark the dataset-specific $t^*$. On CelebA-HQ ($t^* \approx 0.903$) and FFHQ ($t^* \approx 0.894$) the empirical rise begins near $t^*$ with right-edge magnitude $\approx 2(1-t^*)$, matching the prediction. CIFAR-10 ($t^* \approx 0.994$) has almost no ascending-branch region, so its residual right-edge rise reflects other effects, most plausibly model misspecification of the spiked-covariance Gaussian fit as the marginal becomes data-like.}
\label{fig:estimator_error_t}
\end{figure}

\subsection{Closed-Form Variance Decomposition on the Spiked Gaussian}\label{app:exp_time_gap}

This experiment underlies Figure~\ref{fig:ambient_dim_analysis}. It validates Theorem~\ref{thm:time_independence_cost} (equivalently, Theorem~\ref{thm:relative_time_independence_cost}) in a setting where the ambient and signal dimensions can be controlled directly, so the predicted separation of scales between the coupling variance (Term~II) and the time-blindness gap (Term~III) is observable without confounds from model capacity or data distribution.

\paragraph{Setup.}
We use the spiked Gaussian model throughout: $\varepsilon \sim \mathcal{N}(0, I_d)$, $x \sim \mathcal{N}(0, \Sigma)$ with $\Sigma = \mathrm{diag}(\lambda, \dots, \lambda, \sigma^2, \dots, \sigma^2)$ (the first $k$ coordinates are signal with variance $\lambda$, the remaining $d - k$ are isotropic noise with variance $\sigma^2$), $z = (1-t)\varepsilon + tx$, and $u = x - \varepsilon$. Default parameters are $\lambda = 10$, $\sigma^2 = 0.01$, with $t \sim \mathrm{Uniform}[\tau, 1-\tau]$, $\tau = 0.15$. Because $\Sigma$ is diagonal and $\varepsilon, x$ are independent, the pairs $(z_i, u_i)$ are jointly Gaussian and independent across coordinates, so all quantities of interest reduce to sums of per-coordinate integrals and admit either closed forms or low-variance numerical estimators.

\paragraph{Term~I: learnable error.}
The learnable error is the variance of $u = x - \varepsilon$ with no conditioning, which has the closed form
\[
    \mathrm{tr}(\mathrm{Cov}(u)) \;=\; \mathrm{tr}(\Sigma + I_d) \;=\; k(\lambda + 1) + (d - k)(\sigma^2 + 1).
\]
This is the only one of the three terms that admits a fully closed-form integral, since it does not depend on $t$.

\paragraph{Term~II: coupling variance $\mathcal{C}_\pi$.}
Per coordinate, joint Gaussianity of $(z_i, u_i)$ gives the standard regression-variance closed form
\[
    \mathrm{Var}(u_i \mid z_i, t) \;=\; (S_i + 1) - \frac{(tS_i - (1-t))^2}{v_i(t)},
    \qquad
    v_i(t) \;:=\; (1-t)^2 + t^2 S_i,
\]
with $S_i \in \{\lambda, \sigma^2\}$ depending on whether $i$ is a signal or noise coordinate. This expression does not depend on $z_i$ (Gaussian--Gaussian conditioning), so summing over coordinates produces a deterministic function of $t$. We integrate it against $t \sim \mathrm{Unif}[\tau, 1-\tau]$ by the trapezoidal rule on a $2000$-point grid to obtain $\mathcal{C}_\pi$.

\paragraph{Total time-blind variance $\mathbb{E}_z[\mathrm{tr}(\mathrm{Var}(u \mid z))]$.}
This is the irreducible error when $t$ is not observed. Unlike $\mathcal{C}_\pi$ it does not reduce to a closed-form integral, because $\mathrm{Var}(u \mid z)$ depends on the full posterior $p(t \mid z)$. We estimate it with a hybrid Monte-Carlo / grid scheme.

We draw $N = 2 \times 10^5$ outer samples $(t, \varepsilon, x)$ with $t \sim \mathrm{Unif}[\tau, 1-\tau]$, $\varepsilon \sim \mathcal{N}(0, I_d)$, and $x \sim \mathcal{N}(0, \Sigma)$, and form $z = (1-t)\varepsilon + tx$. For each such $z$, we evaluate the posterior on a $2000$-point grid $\{t'\}$ in $[\tau, 1-\tau]$ via
\[
    p(t' \mid z) \;\propto\; \prod_i \mathcal{N}\!\bigl(z_i;\, 0,\, v_i(t')\bigr).
\]

From this posterior, we compute the conditional mean of each coordinate of $u$ as
\[
    \mathbb{E}[u_i \mid z] \;=\; z_i \sum_{t'} \alpha_i(t')\, p(t' \mid z),
    \qquad
    \alpha_i(t') \;:=\; \frac{t' S_i - (1-t')}{v_i(t')},
\]
and the conditional second moment as
\[
    \mathbb{E}[u_i^2 \mid z] \;=\; \sum_{t'} \bigl[\mathrm{Var}(u_i \mid z, t') + (\alpha_i(t')\, z_i)^2\bigr]\, p(t' \mid z).
\]
The per-coordinate conditional variance follows by subtraction:
\[
    \mathrm{Var}(u_i \mid z) \;=\; \mathbb{E}[u_i^2 \mid z] - \bigl(\mathbb{E}[u_i \mid z]\bigr)^2.
\]

Summing across coordinates and averaging across the $N$ samples of $z$ gives the estimator of $\mathbb{E}_z[\mathrm{tr}(\mathrm{Var}(u \mid z))]$. All arithmetic is performed in double precision to avoid numerical cancellation when Term~III is orders of magnitude smaller than Term~II.

\paragraph{Term~III: time-blindness gap $\Delta$.}
The time-blindness gap is recovered by the law of total variance as the residual between the total time-blind variance and the coupling variance: $\Delta = \mathbb{E}_z[\mathrm{tr}(\mathrm{Var}(u \mid z))] - \mathcal{C}_\pi$. This decomposition lets us extract $\Delta$ without needing a separate Monte-Carlo estimator for it.

\paragraph{Sweep.}
We sweep both axes of the $(d, k)$ grid simultaneously: $k \in \{1, 2, 4, \dots, 1024\}$ and $d \in \{2, 4, 8, \dots, 1024\}$ (log-spaced powers of two), restricted to $k < d$ since the construction requires a non-empty noise subspace. At each $(d, k)$ we recompute all three terms from scratch. From this $(d, k)$ table we extract three slices for Figure~\ref{fig:ambient_dim_analysis}: the column at fixed $d = 1024$ sweeping $k$ (left panel), the entire grid as a heatmap of Term~III (center panel), and the row at fixed $k = 64$ sweeping $d$ (right panel).

\paragraph{Results (Figure~\ref{fig:ambient_dim_analysis}).}
The right panel reproduces the asymptotic gap-collapse: at $k = 64$ fixed, Term~III decays from $\sim\!2$ at the smallest admissible $d$ to below $\sim\!0.5$ by $d = 1024$, while Terms~I and II stay bounded over the same sweep. The left panel shows the complementary regime: at $d = 1024$ fixed, all three terms grow with $k$, but Term~III grows fastest because the noise subspace $d - k$ is shrinking, sharpening the prediction that the variance-based identifiability mechanism degrades as $k \to d$. The center heatmap shows both regimes simultaneously: $\Delta$ is small along the upper-right edge ($k \ll d$) and large along the diagonal ($k \approx d$), tracing the same geometry across the $(d, k)$ plane. Together, these slices give a quantitative signature of Theorem~\ref{thm:time_independence_cost}.

\end{document}